\documentclass[nohyperref]{article}

\usepackage{microtype}
\usepackage{graphicx}
\usepackage{subfigure}
\usepackage{booktabs} 

\usepackage{thmtools}
\usepackage{thm-restate}

\usepackage{setspace}
\usepackage{multirow}

\usepackage{hyperref}
\usepackage{makecell}

\usepackage[accepted]{icml2022}

\usepackage{amsmath}
\usepackage{amssymb}
\usepackage{mathtools}
\usepackage{amsthm}
\usepackage{bm}

\usepackage[capitalize,noabbrev]{cleveref}

\definecolor{Blue}{rgb}{0.0,0.0,1.0}

\theoremstyle{plain}
\newtheorem{theorem}{Theorem}

\newtheorem*{proposition*}{Proposition}

\newtheorem{definition}[theorem]{Definition}

\theoremstyle{remark}

\usepackage[textsize=tiny]{todonotes}

\usepackage{xcolor}
\definecolor{midgreen}{rgb}{0.1,0.5,0.1}
\definecolor{darkgray}{gray}{0.25}
\definecolor{lightblue}{rgb}{0.25,0.25,1}

\theoremstyle{definition}

\newcommand{\norm}[1]{\ensuremath{\left\| #1 \right\|}}

\newcommand{\inner}[1]{\left \langle {#1} \right \rangle}
\newcommand{\bigo}{\mathcal{O}}
\newcommand{\abs}[1]{\left |#1\right|}
\newcommand{\poly}{\mathrm{poly}}

\def\argmax{\mathop{\rm argmax}}

\def\tr{\mathrm{tr}}

\def\0{{\bm 0}}

\def\b{{\bm b}}

\def\v{{\bm v}}
\def\x{{\bm x}}
\def\y{{\bm y}}

\def\z{{\bm z}}
\def\A{{\bm A}}
\def\B{{\bm B}}
\def\C{{\bm C}}
\def\D{{\bm D}}

\def\G{{\bm G}}

\def\I{{\bm I}}

\def\K{{\bm K}}
\def\L{{\bm L}}
\def\M{{\bm M}}
\def\N{{\bm N}}

\def\Q{{\bm Q}}
\def\R{{\bm R}}
\def\S{{\bm S}}
\def\U{{\bm U}}
\def\V{{\bm V}}
\def\W{{\bm W}}
\def\X{{\bm X}}

\def\Lambda{\boldsymbol{\lambda}}

\def\Rbb{\mathbb{R}}

\def\Lhat{{\widehat{\L}}}

\def\tmix{t_{\mathrm{iter}}}

\icmltitlerunning{Scalable MCMC Sampling for NDPPs}

\begin{document}

\twocolumn[
\icmltitle{Scalable MCMC Sampling for Nonsymmetric Determinantal Point Processes}




\begin{icmlauthorlist}
\icmlauthor{Insu Han}{yale}
\icmlauthor{Mike Gartrell}{criteo}
\icmlauthor{Elvis Dohmatob}{fair}
\icmlauthor{Amin Karbasi}{yale}
\end{icmlauthorlist}

\icmlaffiliation{yale}{Yale University}
\icmlaffiliation{criteo}{Criteo AI Lab, Paris, France}
\icmlaffiliation{fair}{Facebook AI Lab, Paris, France}



\icmlcorrespondingauthor{Insu Han}{insu.han@yale.edu}
\icmlcorrespondingauthor{Mike Gartrell}{m.gartrell@criteo.com}

\icmlkeywords{Machine Learning, ICML}

\vskip 0.3in
]



\printAffiliationsAndNotice{}  

\begin{abstract}
A determinantal point process (DPP) is an elegant model that assigns a probability to every subset of a collection of $n$ items.  While conventionally a DPP is parameterized by a symmetric kernel matrix, removing this symmetry constraint, resulting in nonsymmetric DPPs (NDPPs), leads to significant improvements in modeling power and predictive performance.  Recent work has studied an approximate Markov chain Monte Carlo (MCMC) sampling algorithm for NDPPs restricted to size-$k$ subsets (called $k$-NDPPs). However, the runtime of this approach is quadratic in $n$, making it infeasible for large-scale settings.  In this work, we develop a scalable MCMC sampling algorithm for $k$-NDPPs with low-rank kernels, thus enabling runtime that is sublinear in $n$.  Our method is based on a state-of-the-art NDPP rejection sampling algorithm, which we enhance with a novel approach for efficiently constructing the proposal distribution.  Furthermore, we extend our scalable $k$-NDPP sampling algorithm to NDPPs without size constraints.  Our resulting sampling method has polynomial time complexity in the rank of the kernel, while the existing approach has runtime that is exponential in the rank.  With both a theoretical analysis and experiments on real-world datasets, we verify that our scalable approximate sampling algorithms are orders of magnitude faster than existing sampling approaches for $k$-NDPPs and NDPPs.
\end{abstract}

\renewcommand{\arraystretch}{1.2}
\begin{table*}[t]
    \vspace{-0.1in}
\caption{Summary of recent NDPP sampling algorithms. The sampling time of \citet{anonymous2022scalable}'s work assumes an orthogonal constraint on the kernel. Here, $n$ is the size of ground set, $d$ is the rank of the kernel, $k$ refers to the size of sampled set ($k\leq d \ll n$) and $\alpha>0$ is a data-dependent factor. \citet{alimohammadi2021fractionally} showed that $\tmix=\poly(k)$ guarantees the convergence of MCMC sampling, where $\kappa > 0$ is a condition number of the NDPP kernel component (see \cref{thm-total-runtime} for details).}
	\vspace{0.05in}
	\centering
	\scalebox{0.85}{
	\setlength{\tabcolsep}{11pt}
	\begin{tabular}{llll}
		\toprule
		{Algorithm} & {Task} & {Preprocessing Time} & {Sampling Time} \\
		\midrule
		Cholesky-based Exact~\cite{poulson2019high} & NDPP  & $-$ & $\bigo(n \cdot d^2)$ \\
		Rejection-based Exact~\cite{anonymous2022scalable} & NDPP & $\bigo(n \cdot d^2)$ & $\bigo((\log n \cdot k^3 + k^4 + d) \cdot (1+\alpha)^d)$\\
		Na\"ive MCMC~\cite{alimohammadi2021fractionally}  & $k$-NDPP / NDPP & $-$ & $\bigo(n^2 \cdot k^3 \cdot \tmix)$\\
		Scalable MCMC~(\textbf{This work}) & $k$-NDPP / NDPP & $\bigo(n \cdot d^2)$ & $\bigo((\log n \cdot d^2 + d^3) \cdot (1+\kappa)^2 \cdot \tmix)$\\
		\bottomrule
	\end{tabular}
	}\label{tab:summary-recent-ndpp-algorithms}
\end{table*}

\section{Introduction}

Determinantal Point Processes (DPPs) are probability distributions defined on the set of all subsets of a collection of $n$ items. They have been applied to a variety of fundamental machine learning problems, including robustness learning~\cite{pang2019improving}, reinforcement learning~\cite{yang2020multi}, and bandit optimization~\cite{kathuria2016batched}, among many others.
While conventionally a DPP is parameterized by a symmetric kernel matrix, \citet{gartrell2019nonsym} showed that any nonsymmetric and positive semidefinite matrix can define a valid DPP, which they refer to as a nonsymmetric DPP (NDPP). In addition, they established a number of useful properties of NDPPs.  For example, NDPPs are able to capture both positive and negative correlations among items, while symmetric DPPs can only represent negative correlations, leading to significant improvements in modeling power and predictive performance. 

Recent works have proposed efficient algorithms for various NDPP tasks, including learning~\cite{gartrell2020scalable}, MAP inference~\cite{anari2021sampling}, and sampling~\cite{anonymous2022scalable}, where the NDPP kernel is given by a low-rank factorization. In this paper we focus on developing an efficient sampling algorithm for NDPPs restricted to size $k$ subsets, called $k$-NDPPs. Such size-constrained DPPs are often more practical in applications such as video summarization~\cite{sharghi2018improving}, mini-batch optimization~\cite{zhang2017determinantal}, document summarization~\cite{dupuy16} and coreset sampling~\cite{tremblay2019determinantal}.
The only existing approach for $k$-NDPP sampling is an approximate method based on Markov chain Monte Carlo (MCMC) sampling ~\cite{alimohammadi2021fractionally,anari2021sampling}. 
The algorithm is based on a random walk, where in every iteration a pair of items is exchanged with some probability. These prior works primarily focused on the number of iterations required for convergence, and proved that with time polynomial in $k$, the sampling algorithm converges to the $k$-NDPP target distribution. However, each transition step needs time quadratic in $n$, making this approach infeasible for large-scale settings.

\vspace{-0.05in}
\subsection{Contributions}
In this work, we develop a scalable MCMC sampling algorithm for $k$-NDPPs 
with low-rank kernels.
In particular, we accelerate the transition step of the MCMC sampling algorithm so that it runs in sublinear (polynomial-logarithmic) time in $n$.  We first show that this step is equivalent to sampling a subset of size $2$ from a conditional NDPP.  To achieve fast $2$-NDPP sampling, 
we make use of a state-of-the-art NDPP rejection sampling algorithm~\cite{anonymous2022scalable}, which we enhance with a novel approach for efficiently constructing a symmetric DPP this is used for the proposal distribution. When the NDPP kernel is given by a rank-$d$ factorization ($d \ll n$), the proposal DPP kernel can be constructed in only $\bigo(d^3)$ time.  This type of proposal kernel is similar to a personalized version of the DPP kernel~\cite{gillenwater2019tree,han2020aistats}, which consists of a global features matrix $\X\in\Rbb^{n \times d}$ and a personalization matrix $\U \in \Rbb^{d \times d}$.  This proposal distribution changes in every transition of the MCMC sampling, however according to our construction it suffices to update only the matrix $\U$.  
This allows us to utilize a fast tree-based DPP sampling algorithm~\cite{gillenwater2019tree} suitable for personalized DPPs.  This tree-based algorithm requires us to build a binary tree based on the global features, as a one-time preprocessing step. After preprocessing, the sampling algorithm runs in time that is logarithmic in $n$.  This makes a single iteration of the rejection sampling much faster. We further prove that the number of rejections does not depend on the dimensions of the NDPP kernel, but on some spectral bounds of the kernel.  As a consequence, our MCMC sampling algorithm for $k$-NDPPs runs in logarithmic time in $n$, and polynomial time in both $d$ and $k$.
To the best of our knowledge, this is the first work on a sublinear time algorithm for $k$-NDPP sampling.
In our experiments, we observe that our proposed algorithm runs orders of magnitude faster than the existing sampling approach, which for kernels learned from some datasets does not terminate within 10 days.

Furthermore, we extend our sampling algorithm to size-unconstrained NDPPs. 
The resulting algorithm has polynomial time complexity in the rank $d$ of the kernel, while the existing sampling algorithm for NDPPs~\cite{anonymous2022scalable} has runtime that is exponential in $d$.
Through theoretical analysis and experiments on real-world datasets, we show that our approximate sampling algorithm is orders of magnitude faster than the fastest existing sampling approach for $k$-NDPPs, and up to an order of magnitude faster for NDPPs.
The source code for our NDPP sampling algorithms is publicly available at \url{https://github.com/insuhan/ndpp-mcmc-sampling}.

\subsection{Related Work}
Fast sampling algorithms for symmetric DPPs have been extensively studied, including a tree-based algorithm~\cite{gillenwater2019tree}, and an intermediate sampling method~\cite{derezinski2019fastdpp}.  These methods commonly require a one-time preprocessing step, with the subsequent sampling procedure running in time that is sublinear in the size of the ground set $n$. \citet{celis2017complexity} studied a polynomial time sampling algorithm under partition constraints. For unconstrained-size NDPP sampling, \citet{poulson2019high} developed the Cholesky-based sampling algorithm, which runs in time $\bigo(n^3)$ for general kernels. Recently, \citet{anonymous2022scalable} showed that with a rank-$d$ kernel decomposition, the runtime of the Cholesky-based algorithm can be reduced to $\bigo(nd^2)$.  Moreover, they propose a tree-based rejection sampling algorithm for NDPPs that combines previous work for fast sampling of symmetric DPPs with an efficient approach for constructing the proposal distribution.  However, although the sampling process has runtime that is sublinear in $n$, they show that the average number of rejections is exponential in $d$, which can be problematic in general.  For $k$-NDPPs, to the best of our knowledge, there is no prior work on an efficient algorithm for exact sampling. The only existing approach is an approximate MCMC sampling algorithm~\cite{alimohammadi2021fractionally}, which has runtime that is quadratic in $n$. We summarize these $k$-NDPP and NDPP sampling algorithms in \cref{tab:summary-recent-ndpp-algorithms}.

\section{Background}

\paragraph{Notation.} The set of first $n$ positive integers is denoted by $\{1,\dots, n\}\coloneqq[n]$.  For a finite set $S$, we denote by $\binom{S}{k}$ the collection of all $k$-element subsets of a set $S$.   We use $\I_d$ for the $d$-by-$d$ identity matrix and drop the subscript when it is clear from the context.  For a matrix $\X \in \Rbb^{m \times n}$ and indices $A\subseteq[m],B\subseteq [n]$, we use $\X_{A,B} \in \Rbb^{|A| \times |B|}$ to denote a submatrix of $\X$ whose rows and columns are indexed by $A$ and $B$, respectively. We write $\X_{:,B} \coloneqq \X_{[m],B}$ to denote all rows of $\X$, and similarly $\X_{A,:}\coloneqq\X_{A,[n]}$ for all columns of $\X$. We denote the largest and smallest singular values of $\X$ by $\sigma_{\max}(\X)$ and $\sigma_{\min}(\X)$, respectively.
We use $\succeq$ to denote the Loewner order, i.e., $\A \succeq \B$ implies $\A - \B$ is positive semidefinite (PSD), and $\mathrm{Diag}$ to denote the direct sum, i.e., $\mathrm{Diag}(\A,\B)=\begin{bsmallmatrix}\A&\bm{0}\\\bm{0}&\B\end{bsmallmatrix}$.

\subsection{Nonsymmetric DPPs}
Given a matrix $\L \in \Rbb^{n \times n}$, a DPP assigns a probability \begin{align}
    \mathcal{P}_{\L}(S) \propto \det(\L_S)
\end{align} to every subset $S$ of $[n]$. Any symmetric and PSD kernel $\L$ guarantees that $\det(\L_S)$ is nonnegative, and therefore admits a DPP.  \citet{gartrell2019nonsym} extended the space of valid DPP kernels and proved that nonsymmetric and PSD $\L$ kernels (i.e., $\L + \L^\top \succeq 0$) can be also used to define a DPP.   An important property of NDPPs is that they can capture both positive and negative correlations, while symmetric DPPs only capture the negative correlations, resulting in significant improvements in modeling power and predictive performance for NDPPs. 

In particular, \citet{gartrell2019nonsym} proposed a kernel construction for NDPPs that combines a symmetric PSD matrix and a skew-symmetric matrix:
\begin{align} \label{eq:ndpp-low-rank}
    \L = \V \V^\top + \B (\D - \D^\top) \B^\top,
\end{align}
where $\V \in \Rbb^{n \times d_1}, \B \in \Rbb^{n \times d_2}$, and $\D \in \Rbb^{d_2 \times d_2}$. For simplicity, we will write $\X := [\V,~\B] \in \Rbb^{n \times d}$, $\W := \mathrm{Diag}(\I, \D-\D^\top) \in \Rbb^{d \times d}$ for $d=d_1+d_2$, and $\L = \X \W \X^\top$. 

We refer to a $k$-NDPP as a NDPP whose support is restricted to size-$k$ subsets of $[n]$.\footnote{Throughout this paper, we assume that $2 \le k \le d \ll n$.} As studied in~\citet[Proposition 5.1]{kulesza11}, the normalization constant of a $k$-NDPP can be computed using the eigenvalues of $\L$.\footnote{This was originally studied for symmetric DPPs, but can be naturally extended to a  nonsymmetric PSD matrix $\L$.} Formally, when $\{ \lambda_i \}_{i=1}^d$ are the nonzero eigenvalues of the rank-$d$ matrix $\L$, it holds that
\begin{align} \label{eq-kdpp-normalization-constant}
    \sum_{S \in \binom{[n]}{k}} \det(\L_S) = \sum_{S \in \binom{[d]}{k}} \prod_{i \in S} \lambda_i \coloneqq e_k(\{\lambda_i\}_{i=1}^d),
\end{align}
where $e_k$ is known as the $k$-th \emph{elementary symmetric polynomial}. 
Note that $\{ \lambda_i \}_{i=1}^d$ are also eigenvalues of $\W \X^\top \X \in \Rbb^{d \times d}$, and therefore one can obtain them from matrix-matrix multiplications and the eigendecomposition, resulting in $\bigo(nd^2)$ runtime. 
In addition, \cref{eq-kdpp-normalization-constant} can be computed in time $\bigo(dk)$ using the following recursive relation:
\begin{align} \label{eq-elementary-symmetric-polynomials-recursion}
    e_k(\{\lambda_i\}_{i=1}^d)=e_k(\{\lambda_i\}_{i=1}^{d-1})+\lambda_d \cdot e_{k-1}(\{\lambda_i\}_{i=1}^{d-1}),
\end{align}
where $e_0(\{\lambda_i\}_{i=1}^d)=1$.  Since every determinant of a principal submatrix of $\L$ is nonnegative, the $e_k$'s for NDPPs are also nonnegative.

\subsection{MCMC Sampling for $k$-NDPPs} \label{sec-kndpp-mcmc-sampling-background}
An MCMC sampling algorithm for a $k$-DPP begins with a subset $S$ selected from $\binom{[n]}{k}$ uniformly at random, and then iteratively updates $S$ with some probability.  For symmetric $k$-DPPs, single-item-exchange Markov chains (i.e., $S$ is replaced with $S \cup \{i\} \setminus \{j\}$ for $i \notin S, j\in S$ in every iteration) can guarantee fast convergence to the approximate target distribution in total variation distance~\cite{li2016fast,anari2016colt,rezaei2019polynomial}.  However, the single-item-exchange chain does not mix well for $k$-NDPPs because they are not negatively dependent, which is a key requirement for fast mixing~\cite{anari2016colt}.

Recent work has shown that when a pair of items $S$ is exchanged, the chain can quickly converge to the target $k$-NDPP distribution~\citep{anari2021sampling,alimohammadi2021fractionally}.  We provide pseudo-code for this MCMC algorithm in \cref{alg-down-up-kndpp}.

\begin{algorithm}[t]
	\caption{MCMC Sampling for $k$-NDPP}  \label{alg-down-up-kndpp}
	\setstretch{1.15}
	\begin{algorithmic}[1]
		\STATE {\bf Input}: $\L \in \Rbb^{n\times n}$, $k \in \mathbb{N}$, $\tmix \in \mathbb{N}$
        \STATE Select $S \in \binom{[n]}{k}$ uniformly at random
		\FOR{$t=1, \dots, \tmix$ } 
		\STATE Select $A \in \binom{S}{k-2}$ uniformly at random
		\STATE Select $a,b \in [n]$ with probability $\propto \det(\L_{A \cup \{a,b\}})$ \label{alg-down-up-kndpp-bottleneck} \\
		 ($\triangleright$ Run \cref{alg-up-operator-rejection-sampling})
		\STATE{$S \gets A \cup \{a,b\}$}
		\ENDFOR
		\STATE {\bf Return} $S$
	\end{algorithmic}
\end{algorithm}

\citet{alimohammadi2021fractionally} proved that the mixing time, i.e., the minimum number of iterations required to approximate the target distribution within $\varepsilon$ in terms of total variation distance, is bounded by a polynomial in $k$.
\begin{proposition*}[Theorem 11 in~\citep{alimohammadi2021fractionally}]
For any $\varepsilon >0$, a sample $S$ obtained from \cref{alg-down-up-kndpp} with 
\begin{align} \label{eq:mixing-time} 
	\tmix=\bigo\left(k^2 \cdot \log\left(\frac{1}{\varepsilon \cdot \Pr(S_0)}\right)\right),
\end{align}
and randomly chosen subset $S_0 \in \binom{[n]}{k}$, the total variation distance to the target $k$-NDPP distribution is guaranteed to be less than $\varepsilon$.
\end{proposition*}

Note that each iteration of MCMC sampling (line 5 in \cref{alg-down-up-kndpp-bottleneck}) needs to compute determinants of $k$-by-$k$ matrices for $\bigo(n^2)$ candidates, and therefore runs in time $\bigo(n^2 k^3)$. We call this step the ``up operator''.

\section{Scalable MCMC Sampling for $k$-NDPPs}

As mentioned above, the na\"ive up operator, which involves an exhaustive search over the space of possible candidates, requires time complexity that is quadratic in the ground set size $n$.  
This runtime clearly suffers from scalability issues for large $n$. 
In this section, we show how to significantly accelerate the up operator by utilizing the low-rank structure of the kernel matrix. 

We first observe that the up operator is equivalent to sampling a size $2$ subset from the NDPP conditioned on $A$.  Formally, given a low-rank NDPP kernel $\L = \X \W \X^\top$ for $\X \in \Rbb^{n \times d}, \W \in \Rbb^{d \times d}$, and a subset $A\subseteq [n], |A|\leq d$, one can check that $\det\left(\L_{A \cup \{a,b\}}\right) \propto \det(\L^A_{\{a,b\}})$, where $\L^A$ is the kernel of the conditional NDPP on $A$, given by
\begin{align} \label{eq-conditioned-ndpp}
    \L^A := \L - \L_{:,A} \L_A^{-1} \L_{A,:} 
    &= \X \W^A \X^\top,
\end{align}
where $\W^A:=\W-\W\X_{A,:}^\top(\X_{A,:}\W\X_{A,:}^\top)^{-1}\X_{A,:}\W$.
Note that computing $\W^A$ requires a matrix inversion of dimension $|A| \le d$ and matrix-matrix multiplications of dimension $d$, which results in $\bigo(d^3)$ operations in total.  Therefore, the up operator can be seen as sampling a size $2$ subset from the NDPP with kernel $\L^A$. 
In the next section, we present an approach for efficiently sampling from this conditional $2$-NDPP.

\subsection{Up Operator via Rejection Sampling} \label{sec-rejection-based-up-operator}

Our goal is an efficient sampling from a $2$-NDPP whose kernel is given by~\cref{eq-conditioned-ndpp}.
To this end, we utilize recent work on a sublinear-time NDPP rejection sampling algorithm~\cite{anonymous2022scalable}. 

Specifically, given a NDPP kernel $\L^A$, assume that there exists a matrix $\Lhat$ such that
\begin{align} \label{eq-proposal-dpp}
    \det ( \L_S^A ) \leq \det ( \Lhat_S )
\end{align}
for every $S \subseteq [n]$.  The rejection sampling method proceeds as follows: first, draw a sample $S$ from the DPP with kernel $\Lhat$ and accept it with probability $\det(\L_S^A) / \det(\Lhat_S)$, otherwise repeat the draws until $S$ is accepted.  The resulting sample $S$ has probability proportional to $\det(\L_S^A)$.  The distribution from which we actually draw a sample (i.e., the DPP with $\Lhat$) is called the \emph{proposal distribution}.
Furthermore, if $\Lhat$ is symmetric, one can make use of several symmetric DPP sampling algorithms.
In particular, we adopt a sublinear-time tree-based method~\cite{gillenwater2019tree} for our scalable MCMC sampling algorithm, which we describe in more detail in~\cref{sec-tree-based-sampling}.

\citet{anonymous2022scalable} provided a proposal distribution with kernel $\Lhat$, based on a spectral decomposition of $\X \W^A \X^\top$, and shows that it satisfies~\cref{eq-proposal-dpp}. When $\L$ is given by a rank-$d$ factorization, this spectral decomposition has a runtime of $\bigo(nd^2)$.  However, this complexity makes the cost of the preprocessing steps for the sampler dominant when the subsequent sampling from the DPP with $\Lhat$ is performed in sublinear-time in $n$ (e.g., using tree-based sampling).  Thus, we would not fully utilize the advantages of a scalable DPP sampling algorithm. We resolve this issue by developing a more efficient procedure for constructing the proposal DPP.

\begin{algorithm}[t]
    \setstretch{1.2}
	\caption{Up Operator via Rejection Sampling} \label{alg-up-operator-rejection-sampling}
	\begin{algorithmic}[1]
		\STATE{ {\bf Input:} $A \subseteq [n]$, $\X\in\Rbb^{n\times d}, \W\in\Rbb^{d \times d}$}
		\STATE $\W^A \gets  \W - \W \X_{A,:}^\top (\X_{A,:} \W \X_{A,:}^\top)^{-1} \X_{A,:} \W$
        \STATE$\{(\sigma_i, \y_i, \z_i)\}_{i=1}^{{d}/{2}}\gets$ Youla decomp. of $\frac{\W^A-\W^A{}^\top}{2}$
		\STATE $\widehat{\W}^A \gets \frac{\W^A + \W^A{}^\top}{2} + \sum_{i=1}^{d/2} \sigma_i \left( \y_i \y_i^\top + \z_i \z_i^\top\right)$
		\WHILE{\text{true}}
        \STATE Sample $\{a,b\}$ with prob. $\propto \det( [\X \widehat{\W}^A \X^\top]_{\{a,b\}})$ \\
        ($\triangleright$ Run \cref{alg-tree-based-kdpp} )
        \IF{$\mathcal{U}([0,1]) \le \frac{\det\left( [\X \W^A \X^\top]_{\{a,b\}} \right)}{\det\left([\X \widehat{\W}^A \X^\top]_{\{a,b\}}\right)}$}
		\STATE {\bf Return } $\{a,b\}$
        \ENDIF
		\ENDWHILE
	\end{algorithmic}
\end{algorithm}

Our key idea is to apply a similar spectral decomposition approach for computing the $d$-by-$d$ matrix $\W^A$, which allows us to compute the proposal DPP kernel in time $\bigo(d^3)$.
More specifically, we begin with writing the spectral decomposition of the skew-symmetric matrix $\frac{\W^A -{\W^A}^\top}{2}$ as
\begin{align} \label{eq-youla-WA}
    \frac{\W^A-\W^A{}^\top}{2} = \sum_{i=1}^{d/2} \begin{bmatrix} \y_i~~\z_i \end{bmatrix}  \begin{bmatrix} 0 & \sigma_i \\ -\sigma_i & 0\end{bmatrix} \begin{bmatrix} \y_i^\top \\ \z_i^\top \end{bmatrix},
\end{align}
where $\{\y_i,\z_i\}_{i=1}^{d/2}$ is a set of eigenvectors, and the $\sigma_i$'s are the nonnegative eigenvalues. 
The above decomposition is also known as the Youla decomposition~\cite{youla1961normal}.
Given this, we define a symmetric matrix $\widehat{\W}^A$ as follows:
\begin{align} \label{eq-spec-sym}
    \hspace{-0.1in} \widehat{\W}^A \coloneqq \frac{\W^A + \W^A}{2} + \sum_{i=1}^{d/2} \begin{bmatrix} \y_i~~\z_i \end{bmatrix}\begin{bmatrix} \sigma_i & 0 \\ 0 & \sigma_i \end{bmatrix} \begin{bmatrix} \y_i^\top \\ \z_i^\top \end{bmatrix}.
\end{align}
An important property is that every determinant of a principal submatrix of $\widehat{\W}^A$ is equal to or greater than that of $\W^A$, i.e., $\det(\W^A_S) \leq \det(\widehat{\W}^A_S)$ for all $S \subseteq [d]$, as shown in \citep[Theorem 1]{anonymous2022scalable}.
We further prove that this property is preserved under the bilinear transformation $\W^A \rightarrow \X \W^A \X^\top$ for any $\X \in \Rbb^{n \times d}$.

\begin{restatable}{theorem}{thmdppupperbound} \label{thm-ndpp-upper-bound}
	Given $\X \in \Rbb^{n \times d}$ and $\W^A \in \Rbb^{d \times d}$, suppose $\widehat{\W}^A$ is obtained from \cref{eq-spec-sym} with $\W^A$. Then, 
	\begin{align} \label{eq-ndpp-new-proposal}
	    \det([\X \W^A \X^\top]_S) \leq \det([\X \widehat{\W}^A \X^\top]_S)
	\end{align}
	for every $S\subseteq [n]$. In addition, equality holds when $|S|\ge d$.
\end{restatable}
We provide the proof of \cref{thm-ndpp-upper-bound} in \cref{sec-proof-thm-ndpp-upper-bound}. \cref{thm-ndpp-upper-bound} allows us to use rejection sampling, with the kernel $\Lhat:=\X \widehat{\W}^A \X^\top$ as the proposal distribution.

Pseudo-code for the up operator computed using rejection sampling is shown in \cref{alg-up-operator-rejection-sampling}.  
Observe that $\widehat{\W}^A$ can be computed in time $\bigo(d^3)$ (lines 2-4 in \cref{alg-up-operator-rejection-sampling}), because both matrix operations involve matrices with dimension $d$, and the Youla decomposition of $\frac{\W^A -{\W^A}{}^\top}{2}$, have complexities $\bigo(d^3)$.  Therefore, we can build each kernel component for the proposal DPP in time $\bigo(d^3)$. 
This improves the previous method with runtime $\bigo(nd^2)$, since $d \ll n$, and potentially allows us to utilize the sublinear-time sampling algorithm. 
In the next section, we discuss the tree-based $k$-DPP sampling algorithm that uses our proposal DPP.

\subsection{Sublinear-time Tree-based Sampling} \label{sec-tree-based-sampling}

We now focus on sampling  the $2$-DPP with kernel $\Lhat = \X \widehat{\W}^A \X^\top$ (line 6 in \cref{alg-up-operator-rejection-sampling}).
Observe that the matrix $\X\in\Rbb^{n \times d}$ remains unchanged, and only the inner matrix $\widehat{\W}^A \in \Rbb^{d \times d}$ changes in every iteration of the MCMC sampling algorithm. Fortunately, the sublinear-time tree-based DPP sampling algorithm~\cite{gillenwater2019tree} is well suited to this type of kernel structure. We build a binary tree using $\X$, which can be used for $2$-DPP sampling with the kernel $\X \widehat{\W}^A \X^\top$, and then the sampling process is equivalent to $k$ tree traversals with a $d$-by-$d$ query matrix. Consequently, $2$-DPP sampling can be done in time $\bigo(d^2 \log n + d^3)$.

We begin by explaining the workflow for tree-based $k$-DPP sampling, where we set $k$ to 2.
Formally, denote $\U := (\widehat{\W}^A)^{-\frac12}$, and let $\{(\lambda_i, \v_i)\}_{i=1}^d$ be the eigendecomposition of $\U (\X^\top\X) \U$. 
From \citep[Eq. (187)]{kulesza2012determinantal}, the probability of sampling $S \in \binom{[n]}{k}$ from the $k$-DPP with $\Lhat$ can be decomposed into the following
\begin{align}  \label{eq-dpp-decomposition}
    \frac{\det(\Lhat_S)}{e_k(\{\lambda_i\}_{i=1}^d)} = \sum_{E \in \binom{[d]}{k}}  \frac{\prod_{i \in E} \lambda_i}{e_k(\{\lambda_i\}_{i=1}^d)} 
    \cdot \det( \K^E_S ),
\end{align}
where $\K^E := \X\U \left(\sum_{i\in E} \lambda_i^{-1} \v_i\v_i^\top\right)\U\X^\top$, and $e_k$ is the elementary symmetric polynomial defined in \cref{eq-kdpp-normalization-constant}.  We observe that $\K^E$ is a rank-$k$ projection matrix, because
\begin{align}
    \K^E=\sum_{i\in E} \frac{\X\U\v_i}{\sqrt{\lambda_i}} \left(\frac{\X\U\v_i}{\sqrt{\lambda_i}}\right)^\top,
\end{align}
and $\frac{\X\U\v_i}{\sqrt{\lambda_i}}$'s are the eigenvectors of $\Lhat$~\citep[Proposition 3.1]{kulesza2012determinantal}. Any projection matrix can define a DPP with a marginal kernel, called an \emph{elementary} DPP.
\cref{eq-dpp-decomposition} allows the following two-step $k$-DPP sampling procedure: 1) select an index set $E \in \binom{[d]}{k}$ with probability $\frac{ \prod_{i\in E} {\lambda}_i }{e_k(\{\lambda_i\}_{i=1}^d)}$, and then 2) sample a subset $S$ from the elementary DPP with kernel $\K^E$. 
As studied in~\citep[Algorithm 8]{kulesza2012determinantal}, step 1) can be efficiently performed using the recursive property of $e_k$ introduced in~\cref{eq-elementary-symmetric-polynomials-recursion}, resulting in $\bigo(dk)$ runtime. 
Notice that step 2) is a computational bottleneck for $k$-DPP sampling. However, this step can be accelerated using tree-based sampling, which we describe next.

Specifically, 
let $S \subseteq [n]$ be a subset that we wish to sample. For any $Y \subseteq [n]$ and $a \notin Y$ observe that
\begin{align}
    &\mathcal{P}_{\K^E}(a\in S | Y \subseteq S) = \frac{\det(\K^E_{Y \cup \{a\}})}{\det(\K^E_Y)} \nonumber \\
    &=\K^E_{a,a} - \K^E_{a,Y} (\K^E_Y)^{-1} \K^E_{Y,a} 
    =\inner{\Q^Y, \x_a^\top \x_a}, \label{eq-inner-product-probability}
\end{align}
where $\Q^Y:=\M - \M\X_{Y,:}^\top (\X_{Y,:}\M \X_{Y,:}^\top)^{-1}\X_{Y,:}\M$, $\M:=\U \left(\sum_{i\in E} \lambda_i^{-1} \v_i\v_i^\top\right)\U$, and $\x_a \in \Rbb^d$ is the $a$-th row vector in $\X$.  This implies that we can begin with $Y \gets \emptyset$ and iteratively append $a$ to $Y$, where $a$ is selected with the probability described in \cref{eq-inner-product-probability}.  The process of selecting a single element can be done in a divide-and-conquer manner by leveraging a  binary tree structure.

\begin{algorithm}[t]
	\caption{Tree-based $k$-DPP Sampling} \label{alg-tree-based-kdpp}
	\setstretch{1.2}
	\begin{algorithmic}[1]
	    \STATE {\bf Input:} $\X \in\Rbb^{n \times d}, \widehat{\W}^A \in \Rbb^{d \times d}$, $\C=\X^\top\X\in\Rbb^{d \times d}$, tree structure $\mathcal{T}$
	    \STATE $\U \gets (\widehat{\W}^A)^{-\frac12}$
	    \STATE $\{(\v_i,\lambda_i)\}_{i=1}^d \gets $ Eigendecomp. of $\U \C \U^\top$
	    \STATE Select size $k$ subset $E$ with prob. $\propto \prod_{i \in E} \lambda_i$ \\ ($\triangleright$ Run Algorithm 8 in~\cite{kulesza2012determinantal})
	    \STATE $\Q \gets \U \left( \sum_{i \in E} \lambda_i^{-1}~\v_i \v_i^\top \right) \U^\top$
	    \STATE $Y \gets \emptyset$
	    \FOR{$j=1, \dots, k$}
	        \STATE Sample $a$ with probability $\inner{\Q ,\x_a \x_a^\top}$ using the tree structure $\mathcal{T}$
	         ($\triangleright$ Run Algorithm 3 in~\cite{anonymous2022scalable})
	        \STATE $Y \gets Y \cup \{a\}$
	        \STATE $\Q \gets \Q - \Q  \X_{Y,:}^\top \left( \X_{Y,:} \Q \X_{Y,:}^\top \right)^{-1} \X_{Y,:} \Q$
	    \ENDFOR
	    \STATE {\bf Return } $Y$
	\end{algorithmic}
\end{algorithm}

We construct a binary tree where the root contains $[n]$ and assigns a partition $A_\ell, A_r$ of $[n]$ to its left and right nodes. The branching proceeds until $n$ leaf nodes are created. In addition, every non-leaf node contains a $d$-by-$d$ matrix $\sum_{a \in A} \x_{a}^\top \x_{a}$, where $A$ is the stored subset.  Sampling a single element can be done by traversing the tree with the query matrix $\Q^Y$. In every non-leaf node containing a subset $A$, we move down to the left branch with probability 
\begin{align} \label{eq-left-probability}
    \frac{\inner{\Q^Y, \sum_{a \in A_\ell} \x_a^\top \x_a}}{\inner{\Q^Y, \sum_{a \in A} \x_a^\top \x_a}}
\end{align}
or otherwise to the right branch, until we reach a leaf node.  The tree traversal process is repeated for $k$ iterations, because every subset sampled from the elementary DPP has exactly $k$ elements. 
If we construct a binary tree of depth $\bigo(\log n)$, which requires time $\bigo(nd^2)$, then step 2) can run in time $\bigo(k d^2 \log n + k^2 d^2)$. We summarize the tree-based $k$-DPP sampling in \cref{alg-tree-based-kdpp} and provide the overall runtime in \cref{thm-runtime-tree-based-kdpp}.

\begin{restatable}{theorem}{thmruntimetreebasedkdpp} \label{thm-runtime-tree-based-kdpp}
    Given $\X\in\Rbb^{n \times d}$, and symmetric and PSD $\widehat{\W}^A \in \Rbb^{d \times d}$, \cref{alg-tree-based-kdpp} samples a subset from the $k$-DPP with kernel $\L = \X \widehat{\W}^A \X^\top$, and runs in time $\bigo(k d^2 \log n + k^2 d^2 + d^3)$, after a one-time preprocessing step that runs in time $\bigo(nd^2)$.
\end{restatable}

We provide the proof of \cref{thm-runtime-tree-based-kdpp} in \cref{sec-proof-thm-runtime-tree-based-kdpp}.  
The runtime of our tree-based sampling algorithm improves that of previous work~\cite{gillenwater2019tree}, which is $\bigo(k^2 d^2 \log n + d^3)$.
We also remark that the binary tree used in \cite{anonymous2022scalable} is slightly different from ours. They build a tree using the eigenvectors of the kernel, while our tree structure is based on the non-orthogonal features $\X$. This allows our tree to be used for sublinear-time sampling for any DPP with kernel $\X\A\X^\top$, with an arbitrarily symmetric and PSD matrix $\A\in\Rbb^{d \times d}$, as is the case for our MCMC-based $k$-NDPP sampling approach.

We remind the reader that the rejection-based up operator requires sampling from a $2$-DPP (line 6 in \cref{alg-up-operator-rejection-sampling}).  From \cref{thm-runtime-tree-based-kdpp}, sampling from the proposal distribution runs in $\bigo(d^2 \log n + d^3)$ time. 
However, as discussed in \cref{sec-rejection-based-up-operator}, this process is repeated until the sample is accepted.  In the next section, we examine the average number of rejections in \cref{alg-up-operator-rejection-sampling}.

\section{Runtime Analysis}

We first define the ratio of the largest and smallest singular values of the conditional kernel components, which will affect the average number of rejections.
\begin{definition} \label{def-condition-number}
Given $\X \in \Rbb^{n \times d}$ and $\W\in \Rbb^{d \times d}$, such that $\W + \W^\top \succeq 0$ and $A \in \binom{[n]}{k-2}$ for $k \ge 2$, consider $\W^A$ as defined in~\cref{eq-conditioned-ndpp}. Define 
\begin{align*}
    \kappa_A := \frac{\sigma_{\max}(\W^A-\W^{A^\top})}{\min_{Y \in \binom{[n]\setminus A}{2}}\sigma_{\min}([\X(\W^A+\W^{A}{}^\top)\X^\top]_Y)}.    
\end{align*}
and $\kappa := \max_{A \subseteq [n], |A|\leq d-2} \kappa_A$.
\end{definition}
We now provide an upper bound on the average number of rejections in \cref{alg-up-operator-rejection-sampling}.

\begin{restatable}{theorem}{thmnumrejectionsbound} \label{thm-num-rejections-bound}
    Given $\X\in\Rbb^{n \times d}$ and $\W \in \Rbb^{d \times d}$, such that $\W + \W^\top \succeq 0$ and $A \in \binom{[n]}{k-2}$ for $k\geq 2$, consider $\kappa_A$ as in \cref{def-condition-number}.
    Then, the average number of rejections of the rejecion-based up operator ~(\cref{alg-up-operator-rejection-sampling}) is no greater than $( 1 + \sigma_{\max}(\X)^2 ~\kappa_A)^2$.
\end{restatable}

{\it Proof Sketch.} 
First, we observe that the average number of rejections can be expressed as
\begin{align} \label{eq-average-num-rejections}
    \frac{ \sum_{Y \in \binom{[n]\setminus A}{2}} \det( [\X \widehat{\W}^A \X^\top]_{Y})}{ \sum_{Y\in \binom{[n]\setminus A}{2}} \det( [\X {\W}^A \X^\top]_{Y})}.
\end{align}
Instead of bounding the above directly, we consider $\max_{Y\in\binom{[n]\setminus A}{2}} \frac{\det( [\X \widehat{\W}^A \X^\top]_Y)}{\det( [\X {\W}^A \X^\top]_Y)}$, which upper bounds \cref{eq-average-num-rejections}. In addition, observing that the denominator is no less than $\sum_{Y\in\binom{[n]\setminus A}{2}}\det([\X (\frac{{\W}^A+\W^A{}^\top}{2})\X^\top]_{Y})$, we can derive the bound as a determinant of a $2$-by-$2$ symmetric and PSD matrix. This can be bounded by the singular values of the kernel.
A full proof is provided in \cref{sec-proof-thm-num-rejections-bound}. \hfill \qed

We observe that the matrices in the numerator and denominator of the factor $\kappa_A$ in \cref{def-condition-number} are bounded by the largest and smallest eigenvalues among some 2-by-2 matrices (see \cref{eq-condition-number-bound} in \cref{sec-proof-thm-num-rejections-bound}). There is no dependency on $d$ here, and therefore the number of rejections does not depend on either $n$ or $d$.
In \cref{sec-exp-runtime}, we empirically verify that the actual rejection numbers are very small compared to $n$ both for synthetic and real-world datasets.  For example, for the Book recommendation dataset with $n\simeq 10^6$~\cite{wan2018item}, we see only $3$ rejections on average.  This makes our rejection-based MCMC sampling algorithm practical for NDPPs with large $n$.

Putting all of the above together, we provide the overall runtime for our MCMC sampling algorithm for $k$-NDPPs in the following proposition.

\begin{restatable}{proposition}{thmtotalruntime} \label{thm-total-runtime}
    Given $\X\in\Rbb^{n \times d}$ and $\W \in \Rbb^{d \times d}$, such that $\W + \W^\top \succeq 0$ and $k\geq 2$, 
    consider $\kappa$ as defined in \cref{def-condition-number}.
    With a preprocessing step that runs in time $\bigo(nd^2)$,
    \cref{alg-down-up-kndpp} runs in time $\bigo(\tmix~(1+\sigma_{\max}(\X)^2~\kappa)^2~(d^2 \log n + d^3))$ in expectation.
\end{restatable}

\begin{figure*}[t]
\begin{minipage}[t]{0.49\textwidth}
    \begin{figure}[H]
        \centering
        \hspace{-0.15in}
     	\subfigure[$k$-NDPP]{\includegraphics[width=0.5\textwidth]{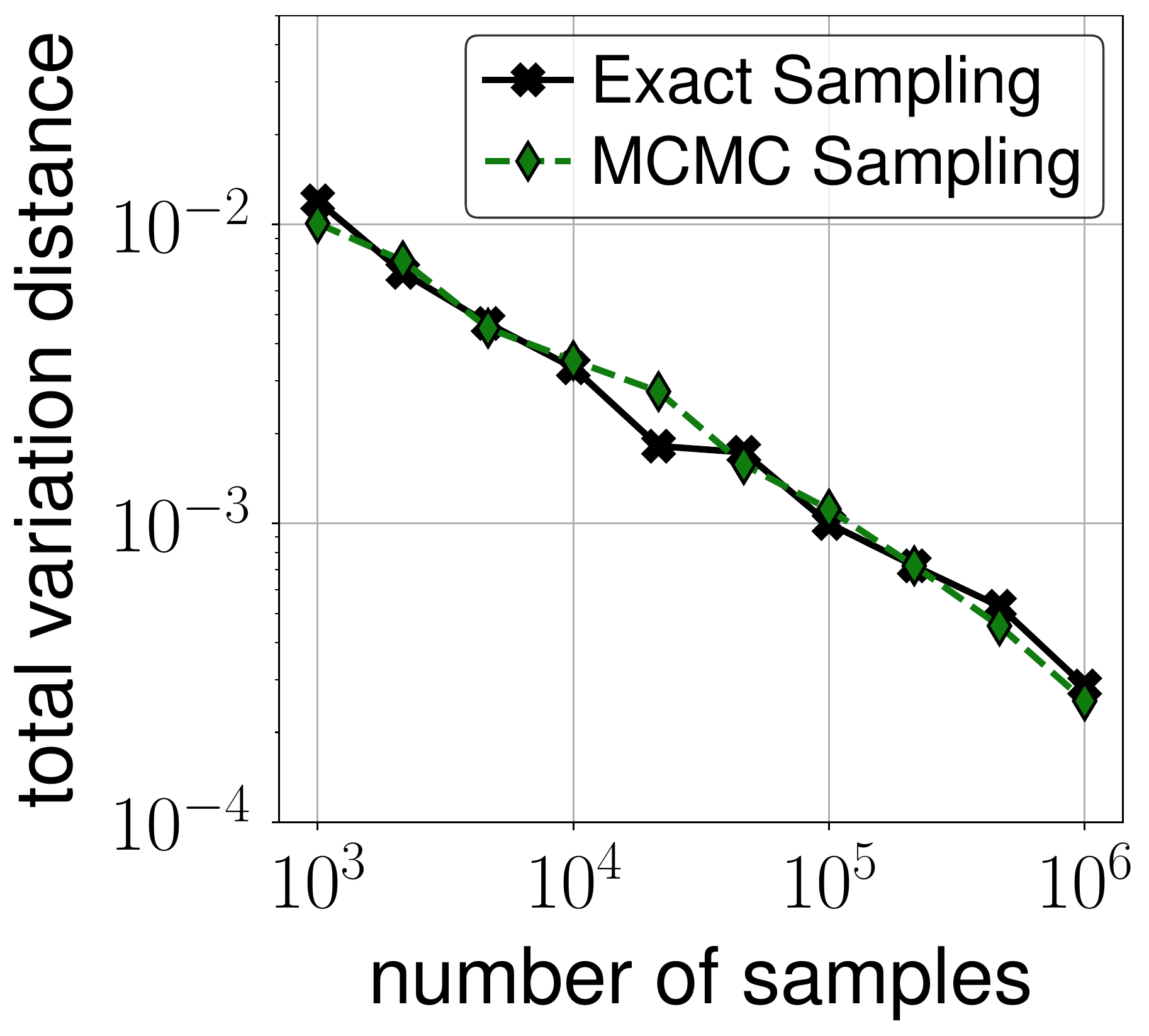}\label{fig:syn-small-sample-quality-kndpp}}
     	\hspace{-0.05in}
     	\subfigure[NDPP]{\includegraphics[width=0.5\textwidth]{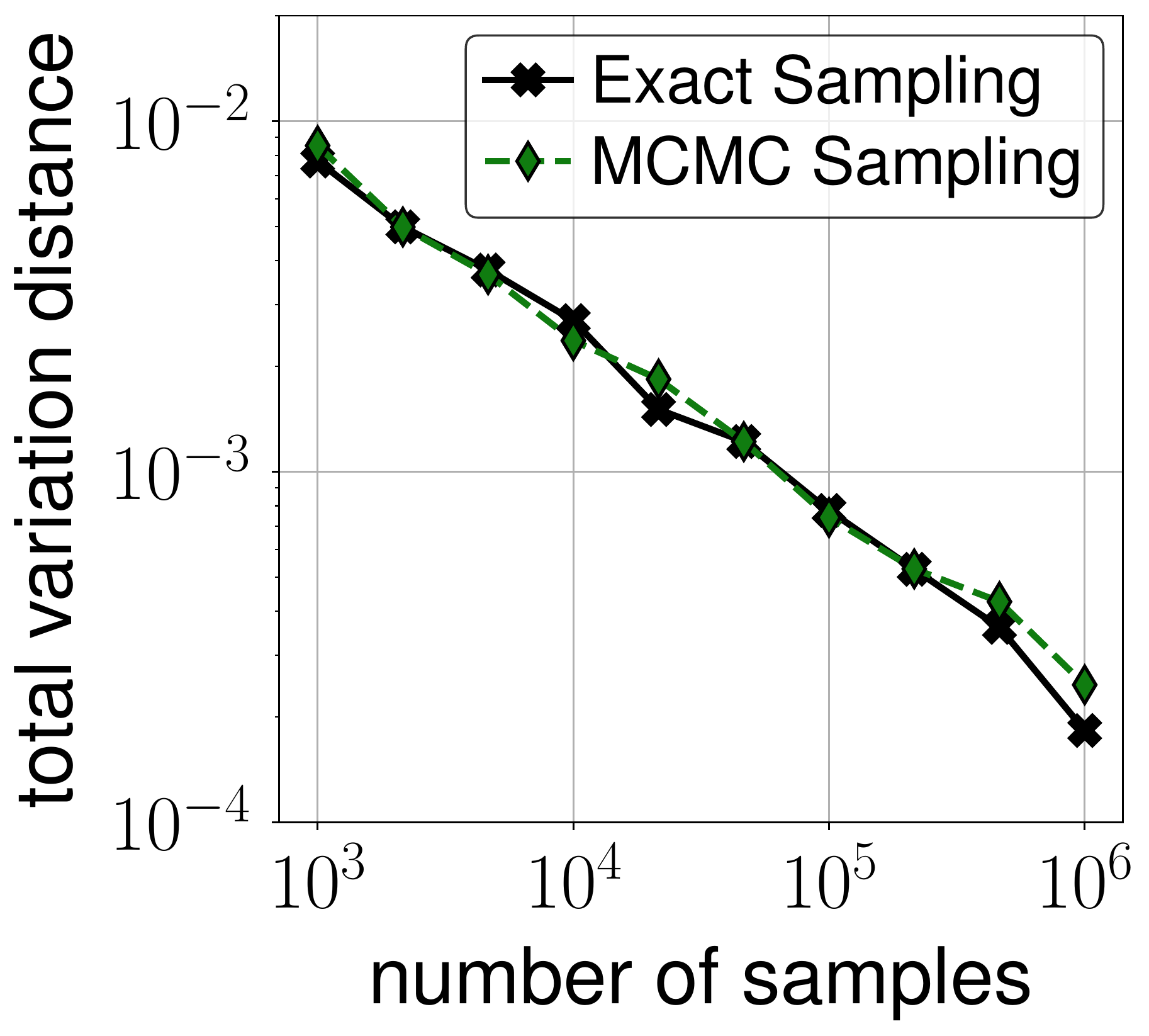}\label{fig:syn-small-sample-quality-ndpp}}
        \vspace{-0.14in}
        \caption{Total variation distance between the exact sampler and our proposed MCMC sampler for (a) a $k$-NDPP and (b) a unconstrained-size NDPP. We use synthetically-generated NDPP kernels with $n=10, d=8, k=5$, and set $\tmix=25$ for our MCMC algorithm.} \label{fig:syn-small-sample-quality}
    \end{figure}
\end{minipage}
\hfill
\begin{minipage}[t]{0.49\textwidth}
    \begin{figure}[H]
        \centering
        \hspace{-0.15in}
     	\subfigure[]{\includegraphics[width=0.5\textwidth]{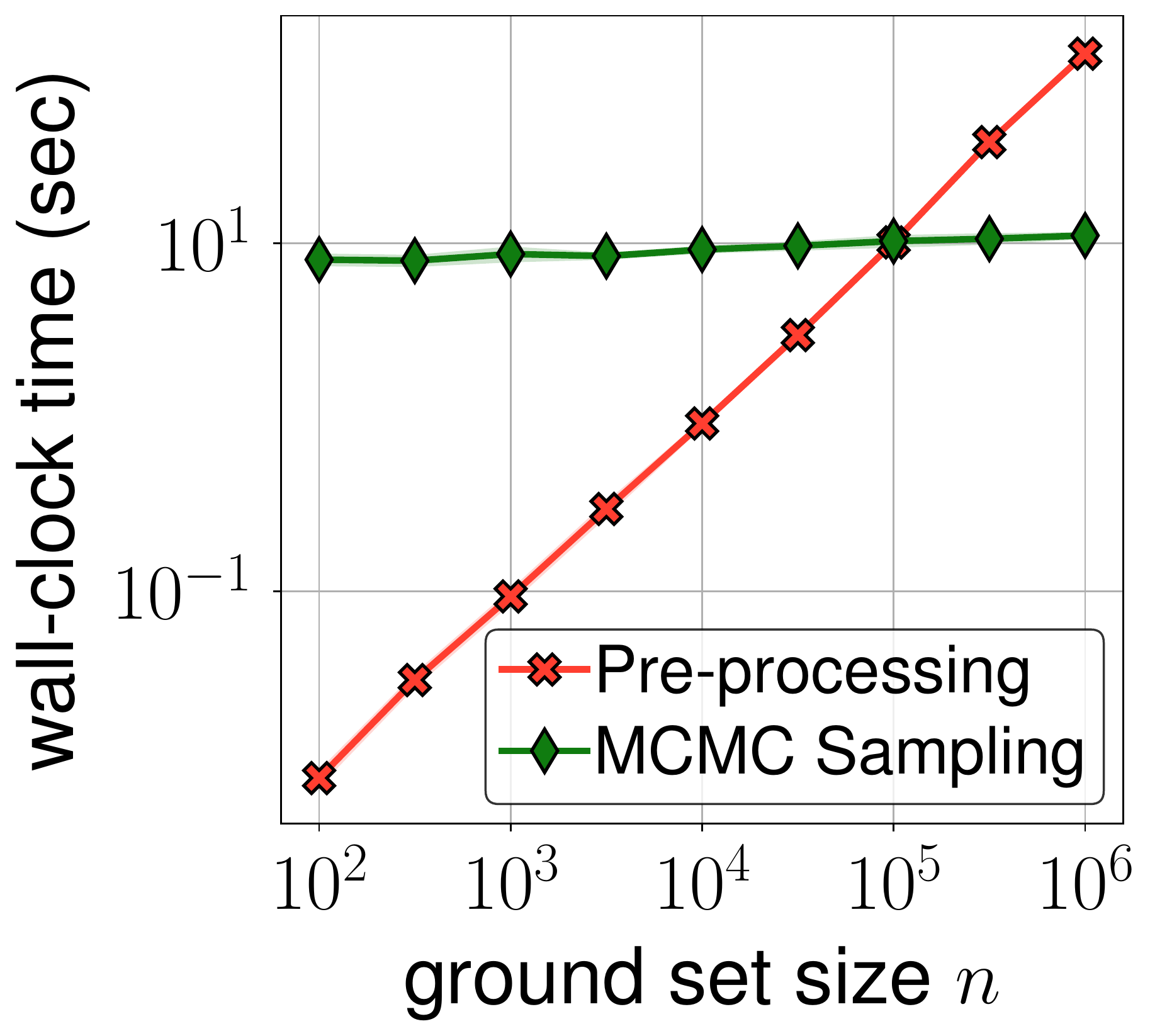}\label{fig:syn-small-runtime-kndpp-n}}
     	\hspace{-0.05in}
     	\subfigure[]{\includegraphics[width=0.5\textwidth]{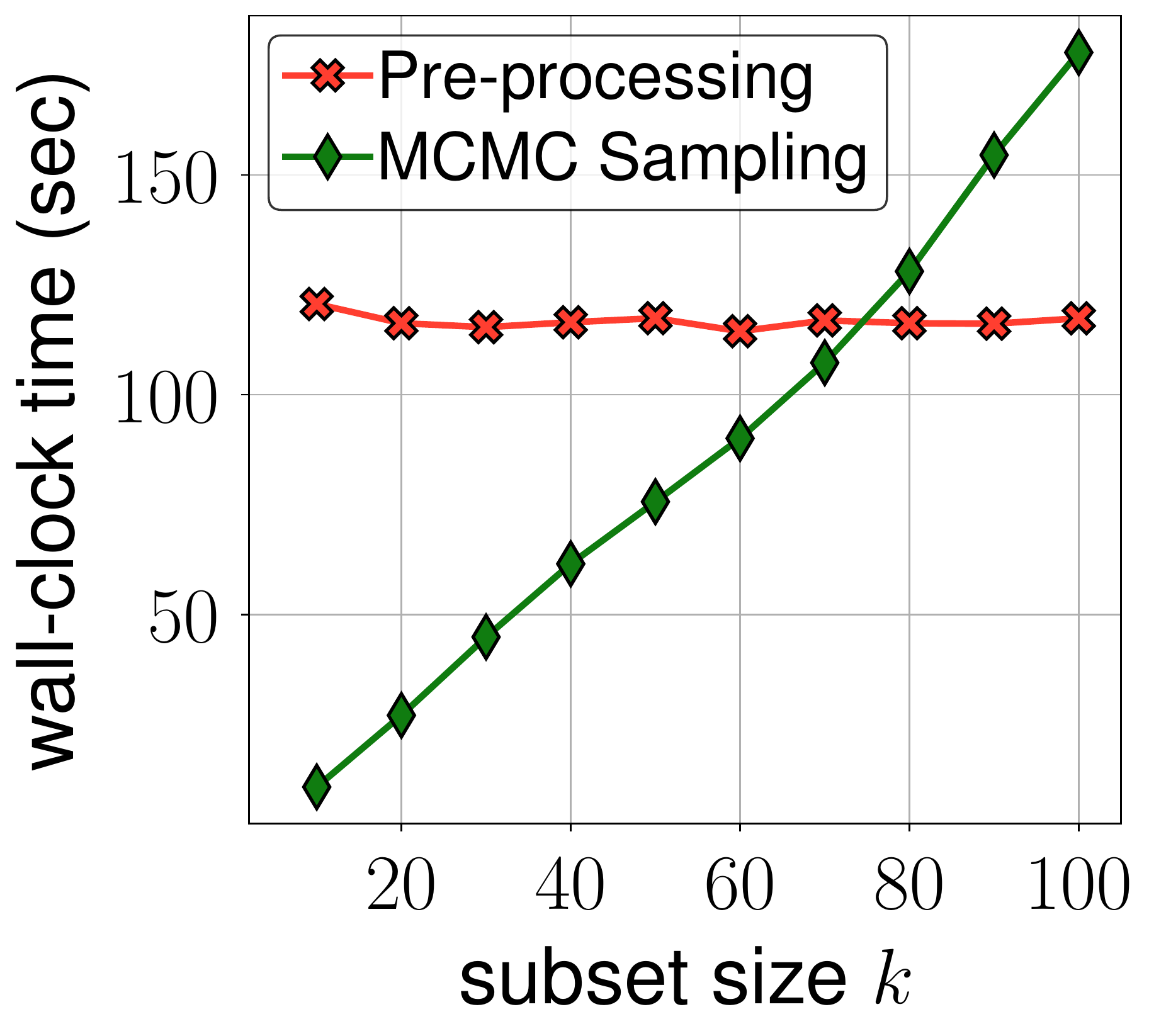}\label{fig:syn-small-runtime-kndpp-k}}
        \vspace{-0.14in}
        \caption{Wall-clock runtime for the preprocessing and sampling steps of our scalable MCMC algorithm, for $k$-NDPPs with synthetic kernels. In (a) we vary $n\in\{10^2,\dots,10^6\}$ and set $k=10$, and in (b) vary $k\in\{10,\dots,100\}$, and set $n=10^6$.} \label{fig:syn-small-runtime-kndpp}
    \end{figure}
\end{minipage}
\end{figure*}

The proof of \cref{thm-total-runtime} can be found in \cref{sec-proof-thm-total-runtime}. Note that the size $k$ only affects the number of MCMC iterations $\tmix$, since each transition of the MCMC algorithm requires sampling from a $2$-NDPP. Moreover, as mentioned in \cref{sec-kndpp-mcmc-sampling-background}, $\tmix = \bigo(k^2 \log\frac{1}{\varepsilon\Pr( S_0)})$ guarantees convergence. Therefore, our MCMC algorithm runs in time that is sublinear in $n$ and polynomial in both $k$ and $d$. 
In \cref{sec-exp-mcmc-convergence}, we compare the MCMC sampling algorithm~(\cref{alg-down-up-kndpp}) to the exact sampler by empirically evaluating the total variation (TV) distance to the ground-truth distribution. We observe that the TV distance of the MCMC sampler with $\tmix=k^2$ decreases as fast as the exact sampler when the number of samples increases.

\section{Extension from $k$-NDPPs to Unconstrained NDPPs}

In this section we show that any $k$-NDPP sampling algorithm can be transformed into an unconstrained-size NDPP sampling algorithm, with a marginal cost for preprocessing.
A simple approach for using a $k$-NDPP sampler to perform NDPP sampling consists of two steps: 1) first, sample a random variable $k\in\{0,1,\dots,d\}$ with probability proportional to the normalization constant of the $k$-NDPP, 
and then 2) run any $k$-NDPP sampling algorithm with the chosen $k$. 
From~\cref{eq-kdpp-normalization-constant}, the normalization constant of a $k$-NDPP is equal to the $k$-th elementary symmetric polynomial $e_k(\{\lambda_i\}_{i=1}^d)$, where $\{\lambda_i\}_{i=1}^d$ are the nonzero eigenvalues of the rank-$d$ kernel. Once we obtain the eigenvalues in $\bigo(nd^2)$ time, the corresponding $e_k$'s can be computed in $\bigo(dk)$ time using \cref{eq-elementary-symmetric-polynomials-recursion}.  
The MCMC sampling algorithm for NDPPs is outlined in~\cref{alg-mcmc-sampling-random-size-ndpp}.

We consider the computation of the $e_k$'s as a preprocessing step, because we re-use these values for drawing subsequent NDPP samples.  The runtime complexity of this preprocessing step is $\bigo(nd^2)$, which is equivalent to the runtime complexity of preprocessing for our sublinear-time MCMC sampling algorithm for $k$-NDPPs.

We describe the overall runtime of \cref{alg-mcmc-sampling-random-size-ndpp} in the following proposition.

\begin{algorithm}[t]
	\caption{MCMC Sampling for NDPP} \label{alg-mcmc-sampling-random-size-ndpp}
	\setstretch{1.2}
	\begin{algorithmic}[1]
		\STATE{ {\bf Input:} $\X \in \Rbb^{n \times d}, \W \in \Rbb^{d \times d}$}
		\STATE $\{(\lambda_i, \v_i)\}_{i=1}^d \gets $ Eigendecomp. of $\W \X^\top \X$
		\STATE Compute elementary symmetric polynomials $\{e_k\}_{k=0}^d$ of $\{\lambda_i\}_{i=1}^d$ ($\triangleright$ Run Algorithm 7 in~\cite{kulesza2012determinantal})
		\STATE Sample $k \in \{0,1,\dots,d\}$ with prob. $\propto e_k$
		\STATE Compute $\tmix$ with the chosen $k$ (e.g., $\tmix = k^2$)
		\STATE Construct a binary tree $\mathcal{T}$ with $\X$
		\STATE $S \gets $ Run \cref{alg-down-up-kndpp} with $\mathcal{T}, \X, \W, k, \tmix$
		\STATE {\bf Return $S$}
	\end{algorithmic}
\end{algorithm}

\begin{restatable}{proposition}{runtimeunconstraintedndpp} \label{thm-runtime-unconstrained-ndpp}
    Given $\X\in\Rbb^{n \times d}$ and $\W \in \Rbb^{d \times d}$, such that $\W + \W^\top \succeq 0$, consider $\kappa$ as in \cref{def-condition-number}.
    With a preprocessing step that runs in time $\bigo(nd^2)$, \cref{alg-mcmc-sampling-random-size-ndpp} runs in time $\bigo(\tmix~(1+\sigma_{\max}(\X)^2~\kappa)^2~(d^2 \log n + d^3))$ in expectation.
\end{restatable}
Previous work on exact NDPP sampling~\cite{anonymous2022scalable} also has runtime that is sublinear in $n$. However, their algorithm has runtime that is exponential in the rank of kernel $d$ (see Theorem 2 therein). In contrast, our MCMC-based approximate sampling algorithm runs in time polynomial in $d$, because of $\tmix=\widetilde{\bigo}(d^2)$.
Such a gap makes our approximate MCMC sampler feasible to run in cases where the exact sampler does not terminate for several days in some real-world settings; see \cref{sec-exp-runtime} for details.

\renewcommand{\arraystretch}{1.2}
\setlength{\tabcolsep}{8pt}
\begin{table*}[t]
\caption{Number of rejections and runtime (in seconds), for sampling and preprocessing, for $k$-NDPP and unconstrained-size NDPP sampling algorithms. Runtimes in the top three rows report sampling times, and the bottom row shows the preprocessing times of our MCMC algorithm. Bold values indicate the fastest runtimes, and ${(*)}$ indicates the expected results for those cases where the sampling algorithm does not terminate within a feasible timeframe.} \label{tab:result-ndpp-sampling}
\vspace{0.08in}
\centering
\scalebox{0.85}{
\begin{tabular}{@{}cccccccc@{}}
	\toprule
	\multicolumn{1}{l}{Task} & Metric & Algorithm & \makecell{{\bf UK Retail} \\ $n=3{,}941$}& \makecell{{\bf Recipe} \\ $n=7{,}993$} & \makecell{{\bf Instacart} \\ $n=49{,}677$} & \makecell{{\bf Million Song} \\ $n=371{,}410$} & \makecell{{\bf Book} \\ $n=1{,}059{,}437$} \\
	\midrule
	\multirow{4}{*}{$k=10$} & \multicolumn{1}{c}{\multirow{2}{*}{Runtime}} & \multicolumn{1}{c}{Exact (Rejection)}  & 406 & {\bf 2.1} & 93.7 & {\bf 0.13} & {\bf 0.46}   \\
    & \multicolumn{1}{c}{} & \multicolumn{1}{c}{MCMC (Ours)}  & {\bf 25.4} & 14.5 & {\bf 21.0} & 9.5 & 23.7 \\
	\cmidrule(l){2-8}
	& \multirow{2}{*}{\# of Rejections} & Exact (Rejection) & 20880 & 79.2 & 3102 & 2.2 & 8.5  \\
	&& MCMC (Ours) & 7.8 & 3.5 & 6.0 & 0.8 & 6.8 \\
	\midrule
    \multirow{4}{*}{$k=50$} & \multicolumn{1}{c}{\multirow{2}{*}{Runtime}} & \multicolumn{1}{c}{Exact (Rejection)}  & ${}^{(*)}$5.11$\times$10$^{12}$ & ${}^{(*)}$9.55$\times$10$^{5}$  & ${}^{(*)}$9.50$\times$10$^{5}$ & ${}^{(*)}$1.45$\times$10$^{12}$ & ${}^{(*)}$4.06$\times$10$^6$  \\
    & \multicolumn{1}{c}{} & \multicolumn{1}{c}{MCMC (Ours)} & {\bf 334} & {\bf 229} & {\bf 242} & {\bf 488} & {\bf 374} \\
	\cmidrule(l){2-8}
	& \multirow{2}{*}{\# of Rejections} & Exact (Rejection) & ${}^{(*)}$2.83$\times$10$^{13}$ & ${}^{(*)}$4.94$\times$10$^6$ & ${}^{(*)}$4.63$\times$10$^6$ & ${}^{(*)}$4.66 $\times$10$^{12}$ & ${}^{(*)}$1.65$\times$10$^{7}$ \\
	&& MCMC (Ours) & 3.8 & 1.3 & 1.6 & 5.4 & 3.2\\
	\midrule
	\multirow{5}{*}{Unconstrained} & \multicolumn{1}{c}{\multirow{3}{*}{Runtime}} & \multicolumn{1}{c}{Exact (Cholesky)}  & {\bf 5.6} & 11.5 & 71.1 & 537 & 1540 \\
	& \multicolumn{1}{c}{} & \multicolumn{1}{c}{Exact (Rejection)} & ${}^{(*)}$1.34$\times$10$^{8}$ & {\bf 1.0} & 1351.6 & ${}^{(*)}$1.89$\times$10$^{10}$ & 1022 \\
    & \multicolumn{1}{c}{} & \multicolumn{1}{c}{MCMC (Ours)}  & 75.3 & 11.8 & {\bf 21} & {\bf 281} & {\bf 80} \\
	\cmidrule(l){2-8}
	& \multirow{2}{*}{\# of Rejections} & Exact (Rejection) & ${}^{(*)}$1.50$\times$10$^9$ & 45.3 & 27941.7 & ${}^{(*)}$6.91$\times$10$^{10}$ & 13924.5 \\
	&& MCMC (Ours) & 6.0 & 3.6 & 5.7 & 7.2 & 9.8\\
        \midrule
	Preprocessing & Runtime & MCMC (Ours) & 1.0 & 2.2 & 14.0 & 30.8 & 74.3 \\
	\bottomrule
\end{tabular}
}
\end{table*}

\section{Experiments}

In this section, we report empirical results for our experiments involving several NDPP sampling algorithms, for NDPPs with and without size constraints. 

\subsection{Convergence of MCMC Sampling} \label{sec-exp-mcmc-convergence}

We first benchmark our MCMC sampling algorithm and compare it to the exact sampler for both $k$-NDPPs and unconstrained-size NDPPs. 
We randomly generate $\V,\B \in \Rbb^{n \times d/2}$, where each entry is sampled from $\mathcal{N}(0,\sqrt{2/d})$; $\D\in\Rbb^{d/2\times d/2}$, where each entry is sampled from $\mathcal{N}(0,1)$; and then construct the NDPP kernel as $\L=\V\V^\top + \B(\D-\D^\top)\B^\top$. 
We collect samples from each sampling algorithm and evaluate the empirical total variation (TV) distance, i.e., $\max_{S} \abs{p(S) - q(S)}$, where $p$ and $q$ correspond to the ground-truth and empirical distributions from the samplers, respectively.
We set $n=10, d=8, k=5$, and draw up to $10^6$ random samples from each sampler.  For our MCMC algorithm, we set $\tmix = k^2$. The results are shown in \cref{fig:syn-small-sample-quality}. We observe that the TV distance of MCMC sampling decays as fast as that of the exact sampler for both $k$-NDPPs and NDPPs. This indicates that setting the number of MCMC iterations to $k^2$ is sufficient for convergence to the target distribution. Therefore, we use $\tmix=k^2$ for all of our experiments.
In \cref{sec-exp-psrf}, we additionally validate our choice for $\tmix$ by evaluating the Potential Scale Reduction Factor (PSRF), commonly used to measure the convergence of the Markov chains~\cite{gelman1992inference}.

\subsection{Runtimes for Synthetic Datasets} 

Next, we report the runtimes of both the preprocessing and sampling steps of our proposed MCMC algorithm.
We generate random NDPP kernels using the same approach described above, and measure the actual runtime in seconds. 
In \cref{fig:syn-small-runtime-kndpp-n}, we vary the size of ground set $n$ from $10^2$ to $10^6$ while fixing $d=100, k=10$. In \cref{fig:syn-small-runtime-kndpp-k}, we vary $k$ from $10$ to $100$ while $n=10^6,d=100$ are fixed. As discussed in~\cref{thm-total-runtime}, we verify that the preprocessing time increases linearly with respect to $n$, and that the sampling time tends to grow sublinearly in $n$. Interestingly, we notice that the sampling times for both $n=10^2$ and $10^6$ are almost identical, at about 10 seconds. This indicates that our algorithm scales well with respect to $n$, and is suitable for large-scale settings.  We also see that our sampling algorithm scales superlinearly with $k$, because the number of MCMC iterations is set to $\tmix=k^2$. 

\subsection{Runtimes for Recommendation Datasets} \label{sec-exp-runtime}

To investigate the practical performance of our proposed sampling algorithms, we apply them to NDPP kernels learned from five real-world recommendation datasets, used in~\cite{anonymous2022scalable}. 
The ground set size $n$ varies from 3{,}941 to 1 million, while the rank of the kernel is generally set to $d=200$ for all datasets.
More details on these datasets can be found in \cref{sec-full-experimental-datasets,sec-full-experimental-setup}. 
We learn the low-rank components of the NDPP kernels, $\V,\B,\D$, using gradient-based maximum likelihood estimation, as described in~\cite{gartrell2020scalable}.\footnote{We use the code from \url{https://github.com/insuhan/nonsymmetric-dpp-sampling} for data preprocessing and NDPP kernel learning.}
We run our algorithms for $k$-NDPPs with sizes $k=10$ and $50$, and unconstrained-size NDPPs, and compare our MCMC algorithms to the exact rejection-based sampling algorithm~\cite{anonymous2022scalable}.
We omit the na\"ive MCMC algorithm~\cite{alimohammadi2021fractionally}, which runs in quadratic time in $n$, from our experiments, because it is over $1{,}000$ times slower than our sampling method for synthetic NDPP kernels with $n=1{,}000$.
For NDPP sampling, we also test the Cholesky-based sampling algorithm~\cite{poulson2019high}, which has linear runtime in $n$. In \cref{tab:result-ndpp-sampling}, we report the runtimes of each sampling algorithm, as well as the number of rejections if the algorithm is based on rejection sampling.

We observe that for the $10$-NDPP, the exact sampling algorithm often runs faster than our MCMC method, e.g., for the Recipe, Million Song, and Book datasets. 
However, for the $50$-NDPP, the exact sampling algorithm results in a very large number of rejections on average, and thus is infeasible for all datasets. 
On the other hand, our MCMC sampler always terminates within a few minutes, running orders of magnitude faster than the exact sampling algorithm.
For NDPP sampling, 
our algorithm is also orders of magnitude faster for the UK Retail and Million Song datasets. 
In \cref{sec:exp-ondpp}, we also apply those sampling algorithms to NDPP kernels learned with an orthogonality constraint (known as ONDPPs), which is tailored to ensure a small number of rejections for NDPP sampling~\cite{anonymous2022scalable}.
These results show that reducing the runtime complexity from exponential to polynomial time can be very important in practice.
Additionally, for NDPP sampling, we see up to a 13 times speedup for our method compared to the linear-time Cholesky-based sampling algorithm. 

\section{Conclusion}

We have shown in this work how to accelerate MCMC sampling for $k$-NDPPs by leveraging a tree-based rejection sampling algorithm. Our proposed sampling algorithm achieves runtime that is sublinear in $n$, and polynomial in $d$ and $k$. We have also extended our scalable $k$-NDPP MCMC sampling approach to NDPP sampling, while preserving the same efficient runtime. Compared to the fastest state-of-the-art exact sampling algorithms for $k$-NDPPs and NDPPs, which have runtime that is quadratic in $n$ or exponential in $d$, respectively, our method makes sampling feasible for large-scale real-world settings by showing significantly faster and more scalable runtimes.

\section*{Acknowledgements}
Insu Han was supported by TATA DATA Analysis
(grant no. 105676). Amin Karbasi acknowledges funding in direct support of this work from NSF (IIS-1845032), ONR (N00014-19-1-2406),  and the AI Institute for Learning-Enabled Optimization at Scale (TILOS).


\bibliographystyle{icml2022}
\bibliography{bibliography}

\clearpage
\onecolumn
\appendix

\section{Additional Details on Experimental Results}

\subsection{Efficient Tree Construction}

Although our MCMC sampler can be very fast for large-scale settings, we do note that consideration of the preprocessing cost is important. Notably, preprocessing requires construction of a binary tree with $\bigo(nd^2)$ memory space, which can be problematic in practice. To alleviate this, we suggest a \emph{fat-leaf tree structure}, where each leaf node contains $B > 1$ elements. 
This reduces the number of nodes in the tree to $\bigo(\frac{n}{B})$, and thus memory space can be reduced to $\bigo(d^2 \frac{n}{B})$.
However, since tree-based sampling returns a leaf node with some probability, 
according to Line 8 in \cref{alg-tree-based-kdpp},
an additional cost for computing the probabilities required for selecting a single item is required, with runtime 
$\bigo(d^2 B)$. Therefore, with this change, the tree-based $2$-DPP sampling runtime becomes $\bigo\left( d^2 \left(\frac{\log n}{B} + B\right) + d^3\right)$.
We set $B=8$ for datasets with $n \ge 10^5$ elements, and observe that the additional runtime overhead is very marginal, while memory consumption is reduced by a factor of $8$.

\subsection{Full Details on Datasets}
\label{sec-full-experimental-datasets}

We perform experiments on the following real-world public datasets:
\begin{itemize}
    \item \textbf{UK Retail:} This dataset~\citep{chen2012data} contains baskets representing transactions from an online retail company that sells all-occasion gifts.  We omit baskets with more than 100 items, leaving us with a dataset containing $19{,}762$ baskets drawn from a catalog of $n = 3{,}941$ products. Baskets containing more than 100 items are in the long tail of the basket-size distribution.
    
    \item \textbf{Recipe:} This dataset~\citep{majumder2019generating} contains recipes and food reviews from Food.com (formerly Genius Kitchen)\footnote{See \url{https://www.kaggle.com/shuyangli94/food-com-recipes-and-user-interactions} for the license for this public dataset.}.  Each recipe (``basket'') is composed of a collection of ingredients, resulting in $178{,}265$ recipes and a catalog of $7{,}993$ ingredients.
    
    \item \textbf{Instacart:} This dataset~\citep{instacart2017dataset} contains baskets purchased by Instacart users\footnote{This public dataset is available for non-commercial use; see \url{https://www.instacart.com/datasets/grocery-shopping-2017} for the license.}.  We omit baskets with more than 100 items, resulting in 3.2 million baskets and a catalog of $49{,}677$ products.
    
    \item \textbf{Million Song:} This dataset~\citep{mcfee2012million} contains playlists (``baskets'') of songs from Echo Nest users\footnote{See \url{http://millionsongdataset.com/faq/} for the license for this public dataset.}. We trim playlists with more than 100 items, leaving $968{,}674$ playlists and a catalog of $371{,}410$ songs.
    
    \item \textbf{Book:} This dataset~\citep{wan2018item} contains reviews from the Goodreads book review website, including a variety of attributes describing the items\footnote{This public dataset is available for academic use only; see \url{https://sites.google.com/eng.ucsd.edu/ucsdbookgraph/home} for the license.}.  For each user we build a subset (``basket'') containing the books reviewed by that user.  We trim subsets with more than 100 books, resulting in $430{,}563$ subsets and a catalog of $1{,}059{,}437$ books.     
\end{itemize}

\subsection{Full Details on Experimental Setup} \label{sec-full-experimental-setup}

\paragraph{NDPP kernel learning.} We use the learning algorithm described in~\cite{gartrell2020scalable}, where we learn the kernel components $\V,\B\in\Rbb^{n \times d/2},\D\in\Rbb^{d/2 \times d/2}$ by minimizing the regularized negative log-likelihood using training example subsets $\{Y_1, \dots, Y_m\}$:
\begin{align}
    \min_{\V, \B, \D}~~-&\frac{1}{m} \sum_{i=1}^m \log \det \left(\V_{Y_i} \V_{Y_i}^\top + \B_{Y_i} (\D-\D^\top) \B_{Y_i}^\top \right) \nonumber \\ 
    &+ \log \det \left(\V \V^\top + \B (\D-\D^\top) \B^\top + \I \right) + \alpha \sum_{i=1}^{n} \frac{\|\v_i\|_2^2}{\mu_i}  + \beta \sum_{i=1}^{n} \frac{\|\b_i\|_2^2}{\mu_i},
\end{align}
where $\v_i$ and $\b_i$ are the $i$-th row vectors of $\V$ and $\B$, respectively.  We also use the training scheme from~\cite{anonymous2022scalable}, where $300$ randomly-selected baskets are held-out as a validation set for tracking convergence during
training, another $2000$ random subsets are used for testing, and the remaining baskets are used for training. Convergence is reached during training when the relative change in validation log-likelihood is below a predetermined threshold. We use the Adam optimizer~\cite{kingma2015adam}; we initialize $\D$ from $\mathcal{N}(0,1)$, and 
$\V$ and $\B$ are initialized from the $\mathcal{U}([0,1])$. We set $\alpha=\beta=0.01$ for all datasets.

\paragraph{ONDPP kernel learning.}  Unlike the NDPP kernel, the orthogonal NDPP (ONDPP) kernel~\cite{anonymous2022scalable} is parameterized as
$\L = \V\V^\top + \B (\D-\D^\top)\B^\top$, where
\begin{align*}
    \D = \mathrm{Diag}\left(
    \begin{bmatrix} 0 & \sigma_1 \\ 0 & 0 \end{bmatrix}, \dots, \begin{bmatrix} 0 & \sigma_{d/2} \\ 0 & 0 \end{bmatrix} 
    \right) \in \Rbb^{d/2 \times d/2}
\end{align*}
and $\sigma_j > 0$. The training objective is 
\begin{align}
    \min_{\V, \B, \{\sigma_j\}_{j=1}^{d/2}} -\frac{1}{m} \sum_{i=1}^m \log\left(\frac{\det(\L_{Y_i})}{\det(\L + \I)}\right) + 
    \alpha \sum_{i=1}^{n} \frac{\|\v_i\|_2^2}{\mu_i} + \beta \sum_{i=1}^{n} \frac{\|\b_i\|_2^2}{\mu_i} + \gamma \sum_{j=1}^{d/2} \log\left(1 + \frac{2\sigma_j}{\sigma_j^2 + 1}\right), 
\end{align}
with constraints $\B^\top \B=\I$ and $\V^\top \B = 0$. To satisfy the first constraint, \citet{anonymous2022scalable} applies QR decomposition on $\B$; for the second constraint, we project $\V$ to the column space of $\B$ by updating $\V \leftarrow \V - \B (\B^\top \B)^{-1}(\B^\top \V)$. We use the regularizer settings from~\citet{anonymous2022scalable}: $\alpha=\beta=0.01, \gamma=0.5$ for the UK Retail dataset, $\alpha=\beta=0.01, \gamma=0.1$ for Recipe, $\alpha=\beta=0.001, \gamma=0.001$ for Instacart, $\alpha=\beta=0.01, \gamma=0.2$ for Million Song, and $\alpha=\beta=0.01, \gamma=0.1$ for Book.

\setlength{\tabcolsep}{8pt}
\begin{table*}[t]
\vspace{-0.09in}
\caption{Number of rejections and runtime (in seconds) for $k$-NDPP and unconstrained-size NDPP sampling algorithms, for ONDPP kernels learned with regularization on the number of NDPP sampling rejections.  Bold values indicate the fastest runtimes, and ${(*)}$ indicates the expected results for those cases where the sampling algorithm does not terminate within a feasible timeframe.}  \label{tab:result-ondpp-sampling}
\vspace{0.1in}
\centering
\scalebox{0.9}{
\begin{tabular}{@{}cccccccc@{}}
	\toprule
	\multicolumn{1}{l}{Task} & Metric & Algorithm & \makecell{{\bf UK Retail} \\ $n$=$3{,}941$}& \makecell{{\bf Recipe} \\ $n$=$7{,}993$} & \makecell{{\bf Instacart} \\ $n$=$49{,}677$} & \makecell{{\bf Million Song} \\ $n$=$371{,}410$} & \makecell{{\bf Book} \\ $n$=$1{,}059{,}437$} \\
	\midrule
	\multirow{4}{*}{$k=10$} & \multicolumn{1}{c}{\multirow{2}{*}{Runtime}} & \multicolumn{1}{c}{Exact (Rejection)}  & {\bf 0.04} & {\bf 0.6} & {\bf 1.9} & {\bf 0.2} & {\bf 0.8} \\
    & \multicolumn{1}{c}{} & \multicolumn{1}{c}{MCMC (Ours)} & 6.5 & 9.0 & 11.8 & 8.3 & 10.7 \\
	\cmidrule(l){2-8}
	& \multirow{2}{*}{\# of Rejections} & Exact (Rejection) & 0 & 12.4 & 36.6 & 2.1 & 13.4 \\
	&& MCMC (Ours) & 0 & 0.9 & 1.3 & 0.3 & 0.9 \\
	\midrule
    \multirow{4}{*}{$k=50$} & \multicolumn{1}{c}{\multirow{2}{*}{Runtime}} & \multicolumn{1}{c}{Exact (Rejection)} & {\bf 0.4 } & $^{(*)}$7.33$\times$10$^{9}$ &  $^{(*)}$1.92$\times$10$^{8}$ & {\bf 99.1} &  $^{(*)}$7.95$\times$10$^6$\\
    & \multicolumn{1}{c}{} & \multicolumn{1}{c}{MCMC (Ours)} & 140.7 & {\bf 450.8} & {\bf 307.9} & 182.9 & {\bf 285.2} \\
	\cmidrule(l){2-8}
	& \multirow{2}{*}{\# of Rejections} & Exact (Rejection) & 0.1 & $^{(*)}$2.08$\times$10$^{8}$ &  $^{(*)}$4.96$\times$10$^{8}$ & 239.7 &  $^{(*)}$1.61$\times$10$^{7}$ \\
	&& MCMC (Ours) & 0.1 & 6.2 & 2.6 & 0.5 & 2.1 \\
	\midrule
	\multirow{4}{*}{Unconstrained} & \multicolumn{1}{c}{\multirow{2}{*}{Runtime}} & \multicolumn{1}{c}{Exact (Rejection)} & {\bf 0.1} & {\bf 0.7} & {\bf 6.0} & {\bf 7.5} & {\bf 2.8} \\
    & \multicolumn{1}{c}{} & \multicolumn{1}{c}{MCMC (Ours)} & 23.1 & 8.3 & 11.7 & 81.3 & 17.1 \\
	\cmidrule(l){2-8}
	& \multirow{2}{*}{\# of Rejections} & Exact (Rejection) & 0.1 & 15.0 & 91.6 & 27.5 & 34.0 \\
	&& MCMC (Ours) & 0.0 & 1.1 & 1.3 & 0.4 & 0.9 \\
	\bottomrule
\end{tabular}
}
\vspace{-0.5 cm}
\end{table*}

\subsection{Additional Experiments with ONDPPs} \label{sec:exp-ondpp}

We apply NDPP sampling algorithms in \cref{sec-exp-runtime} to NDPP kernels learned with an orthogonality constraint (known as ONDPPs), studied in~\cite{anonymous2022scalable}. In particular, these kernels are learned using a regularization mechanism that guarantees a small number of NDPP sampling rejections. Therefore, we expect exact sampling with ONDPP kernels to run very quickly.
\cref{tab:result-ondpp-sampling} shows the results with real-world datasets and ONDPP kernels learned on these datasets. 
As expected, exact ONDPP sampling runs faster than our MCMC approach for unconstrained-size NDPPs. It also runs faster for $10$-NDPP sampling.  However, we see that for $50$-NDPP the expected exact sampling runtimes are over 92 days for three datasets, while our MCMC approach always terminates within a few minutes.
This substantial slowdown for $50$-NDPP sampling results from the runtime being exponential in $k$ for exact sampling, which we are unable to mitigate using regularization during training. 
This suggests that for $50$-NDPP sampling, our scalable MCMC algorithm is the best and only viable choice in practice.

\subsection{Empirical Mixing Time with Potential Scale Reduction Factor (PSRF)} \label{sec-exp-psrf}
We additionally validate the mixing times of our MCMC sampling algorithm (\cref{alg-down-up-kndpp}) using the Potential Scale Reduction Factor (PSRF). PSRF computes the ratio of within-chain and between-chain variances and is frequently used for measuring the empirical mixing times of MCMC samplers. We used the synthetic dataset described in \cref{sec-exp-mcmc-convergence}, and the PSRF implementation in \href{https://www.tensorflow.org/probability/api_docs/python/tfp/mcmc/potential_scale_reduction}{$\mathtt{tensorflow.probability.mcmc}$}, with $100$ independent chains for $n=\{100,200,\dots,3200\}, k=\{2,3,\dots,30\}$, and a fixed $d=20$. Interestingly, as shown in \cref{fig-psrf}, we observe that empirical mixing times increase linearly in $k$ for all choices of $n$. We leave the problem of further improving the mixing time of our NDPP MCMC sampling algorithm for future work.
\begin{figure}[H]
    \centering
    \includegraphics[width=0.35\textwidth]{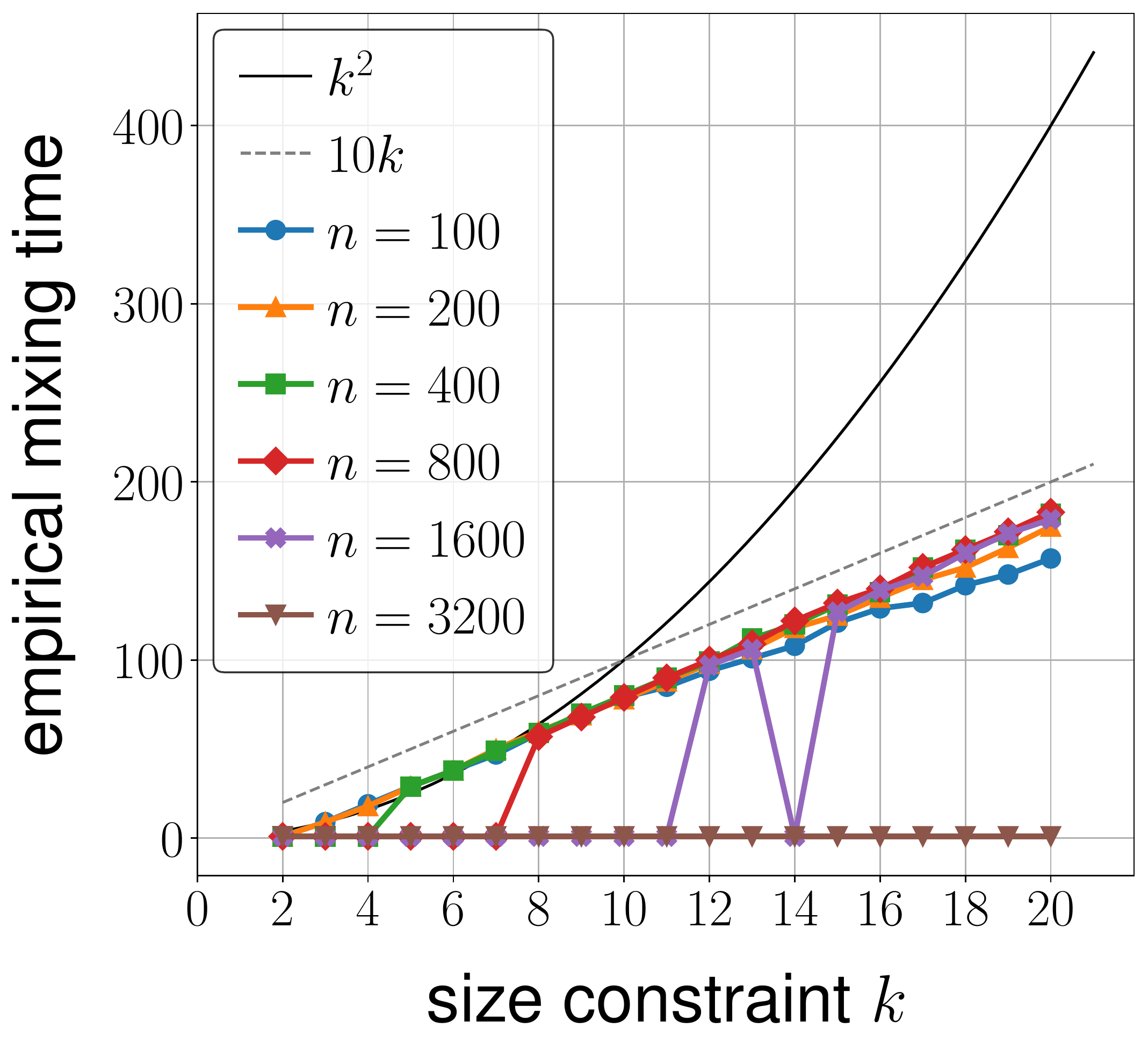}
    \vspace{-0.1in}
    \caption{Empirical mixing time, computed using Potential Scale Reduction Factor (PSRF), for our proposed NDPP MCMC sampling algorithm.} \label{fig-psrf}
\end{figure}

\section{MAP Inference for Initialization}
We observe that the mixing time in \cref{eq:mixing-time} also depends on the initial subset $S_0$. It is desirable to find a size-$k$ subset $S_0$ where $\det(\L_{S_0})$ is as large as possible, and then use $S_0$ as the initial subset in \cref{alg-down-up-kndpp}. This is known as the MAP inference problem for a DPP; that is, $$\argmax_{S \in \binom{[n]}{k}} \det(\L_S).$$
MAP inference for a NDPP is generally known to be NP-hard, and a greedy algorithm is typically used as a heuristic~\cite{gartrell2020scalable}. In particular, \citet{gartrell2020scalable} showed that with a rank-$d$ NDPP kernel, greedy MAP inference runs in $\bigo(nd^2)$ time. Once we find a proper size-$k$ subset $S_0$, we can re-use $S_0$ for drawing subsequent $k$-NDPP samples. Therefore, for a faster mixing time, we utilize greedy MAP inference as a preprocessing step, while preserving the total preprocessing runtime described previously.

Furthermore, this MAP-based initialization approach can also be used for NDPP sampling without size constraints. We note that the greedy algorithm finds elements in the output subset in a sequential way.  In other words, if $\{s_1, \dots, s_d\}$ is the output of the greedy algorithm with size constraint $d$, then the algorithm with size constraint $k\le d$ returns $\{s_1, \dots, s_k\}$. Therefore, for NDPP sampling, we run the greedy algorithm to find a sequence of $d$ items that maximize the determinant of each principal submatrix of size $d$. While running our MCMC NDPP sampler (\cref{alg-mcmc-sampling-random-size-ndpp}), if the size random variable $k$ is selected, then we run the MCMC $k$-NDPP sampling (\cref{alg-down-up-kndpp}) with the chosen $k$ and a subset containing the first $k\le d$ elements in the sequence obtained from the greedy algorithm. In practice, for our experiments in \cref{sec-exp-mcmc-convergence} we observe that our MCMC sampler without greedy initialization shows promising convergence, and thus we omit this procedure in our experiments.

\section{Proofs} \label{sec:proof}

\subsection{Proof of \cref{thm-ndpp-upper-bound}} \label{sec-proof-thm-ndpp-upper-bound}
\thmdppupperbound*

\begin{proof}
For simplicity, we write that $\G:=\X(\frac{\W^A + \W^A{}^\top}{2})\X^\top$ and $\A:=\X(\frac{\W^A - \W^A{}^\top}{2})\X^\top$, so that $\X \W^A \X^\top = \G + \A$. Also, denote $\B:=\X \widehat{\W}^A \X^\top-\G$. 
Since $\G$ is positive semi-definite, for any $S \subseteq [n]$ such that $|S|\leq d$, we have
\begin{align}
	\det([\X \W^A \X^\top]_S) &= \det\left( \G_S + \A_S \right) \nonumber \\ 
	&= \det( \G_S^{1/2} (\I + \G_S^{-1/2} \A_S \G_S^{-1/2}) \G_S^{1/2}) \nonumber \\
	&= \det( \G_S^{1/2}) \cdot \det(\I + \G_S^{-1/2} \A_S \G_S^{-1/2}) \cdot \det(\G_S^{1/2}) 
\end{align}
where $\I$ is the $|S|$-by-$|S|$ identity matrix. Similarly, 
\begin{align}
	\det([\X\widehat{\W}^A \X^\top]_S) = \det( \G_S^{1/2}) \cdot \det(\I + \G_S^{-1/2} \B_S \G_S^{-1/2}) \cdot \det(\G_S^{1/2}).
\end{align}
From Theorem 2.1 in \citep{kulesza2012determinantal}, we have
\begin{align}
	&\det(\I + \G_S^{-1/2} \A_S \G_S^{-1/2}) = \sum_{T \subseteq [|S|]}  \det( [\G_S^{-1/2} \A_S \G_S^{-1/2}]_T),\\
	&\det(\I + \G_S^{-1/2} \B_S \G_S^{-1/2}) = \sum_{T \subseteq [|S|]}  \det( [\G_S^{-1/2} \B_S \G_S^{-1/2}]_T).
\end{align}
Therefore, it is enough to prove that for every $T \subseteq [|S|]$ 
\begin{align}
	\det( [\G_S^{-1/2} \A_S \G_S^{-1/2}]_T) \leq \det( [\G_S^{-1/2} \B_S \G_S^{-1/2}]_T).
\end{align}
Now consider the Youla decomposition on $\frac{\W^A - \W^A{}^\top}{2}$ as in~\cref{eq-youla-WA}, i.e., 
\begin{align}
    \frac{\W^A - \W^A{}^\top}{2} 
    &= \sum_{i=1}^{d/2} \sigma_i \left(\y_i \z_i^\top - \z_i \y_i^\top\right)
    = \V~\mathrm{Diag}\left( \begin{bmatrix} 0 & \sigma_1 \\ -\sigma_1 & 0 \end{bmatrix},\cdots,  \begin{bmatrix} 0 & \sigma_{\frac{d}{2}}, \\ -\sigma_{\frac{d}{2}} & 0 \end{bmatrix}\right)~\V^\top,
\end{align}
where $\V := [\y_1,\z_1\dots,\y_{\frac{d}{2}}, \z_{\frac{d}{2}}]$. Then, it can be written 
\begin{align}
    &\G_S^{-1/2} \A_S \G_S^{-1/2} = \G_S^{-1/2} \X_{S,:} \V \cdot \mathrm{Diag} \left( \begin{bmatrix} 0 & \sigma_1 \\ -\sigma_1 & 0 \end{bmatrix},\cdots,  \begin{bmatrix} 0 & \sigma_{\frac{d}{2}}, \\ -\sigma_{\frac{d}{2}} & 0 \end{bmatrix}  \right) \cdot \V^\top\X_{S,:}^\top \G_S^{-1/2} :=\R, \\
    &\G_S^{-1/2} \B_S \G_S^{-1/2} = \G_S^{-1/2} \X_{S,:} \V \cdot \mathrm{Diag} \left( \sigma_1, \sigma_1, \dots, \sigma_{\frac{d}{2}}, \sigma_{\frac{d}{2}} \right) \cdot \V^\top\X_{S,:}^\top \G_S^{-1/2} := \widehat{\R}.
\end{align}
From Theorem 1 in~\cite{anonymous2022scalable}, it holds that $\det(\R_T) \leq \det(\widehat{\R}_T)$ for all $T \subseteq [|S|]$. This completes the proof of \cref{thm-ndpp-upper-bound}.
\end{proof}

\subsection{Proof of \cref{thm-runtime-tree-based-kdpp}} \label{sec-proof-thm-runtime-tree-based-kdpp}
\thmruntimetreebasedkdpp*
\begin{proof}
The preprocessing for the $k$-DPP sampler includes (1) a binary tree construction based on $\X \in \Rbb^{n \times d}$ and (2) computing $\C = \X^\top \X \in \Rbb^{d \times d}$. Both can be done in $\bigo(nd^2)$ time. Given this preprocessing, \cref{alg-tree-based-kdpp} first performs the eigendecomposition of $\U \C \U^\top$, which requires $\bigo(d^3)$ time. Then, a subset $E\subseteq [d]$ is sampled with probability proportional to $\prod_{i\in E}\lambda_i$ where the $\lambda_i$'s are the eigenvalues of $\U \C \U^\top$. With \citep[Algorithm 8]{kulesza2012determinantal}, this can be done in $\bigo(dk)$ time. 
Next, we need to perform tree-based sampling and query matrix updates for $k$ iterations. 
Since the tree has depth $\bigo(\log n)$, and computing the required probability for moving down the tree takes $\bigo(d^2)$ time, the tree-based sampler requires $\bigo(d^2 \log n)$ time. In addition, computation of the query matrix runs in $\bigo(d^2 k)$. Therefore, the overall runtime of  \cref{alg-tree-based-kdpp} (after preprocessing) is $\bigo(d^3 + kd^2 \log n + k^2 d^2)$. 
This improves the runtime of $\bigo(d^3 + k^2 d^2 \log n)$ from previous work~\cite{gillenwater2019tree}, which uses an alternative probability formulation for the tree traversal in \cref{eq-left-probability} that needs several matrix multiplications in every tree node, resulting in $\bigo(k^2 d^2 \log n)$ runtime.  In our algorithm, these matrix multiplications are computed with a query matrix, outside of the tree traversal.
\end{proof}

\subsection{Proof of \cref{thm-num-rejections-bound}} \label{sec-proof-thm-num-rejections-bound}
\thmnumrejectionsbound*

\begin{proof}
Let $p$ be the probability distribution of the target $2$-NDPP with kernel $\X \W^A \X^\top$, and $q$ be that of the proposal $2$-DPP with kernel $\X\widehat{\W}^A \X^\top$. For every $S \in \binom{[n]\setminus A}{2}$, it holds that
\begin{align*}
    p(S)&=\frac{\det([\X\W^A\X^\top]_S)}{\sum_{\{a,b\}\in \binom{[n]\setminus A}{2}} \det([\X\W^A\X^\top]_{\{a,b\}})}\\
    &\leq \frac{\det([\X\widehat{\W}^A\X^\top]_S)}{\sum_{\{a,b\}\in \binom{[n]\setminus A}{2}} \det([\X\W^A\X^\top]_{\{a,b\}})} \\
    &= \frac{\sum_{\{a,b\}\in \binom{[n]\setminus A}{2}} \det([\X\widehat{\W}^A\X^\top]_{\{a,b\}})}{\sum_{\{a,b\}\in \binom{[n]\setminus A}{2}} \det([\X\W^A\X^\top]_{\{a,b\}})} \cdot \frac{\det([\X\widehat{\W^A}\X^\top]_S)}{\sum_{\{a,b\}\in \binom{[n]\setminus A}{2}} \det([\X\widehat{\W}^A\X^\top]_{\{a,b\}})}\\
    &= \frac{\sum_{\{a,b\}\in \binom{[n]\setminus A}{2}} \det([\X\widehat{\W}^A\X^\top]_{\{a,b\}})}{\sum_{\{a,b\}\in \binom{[n]\setminus A}{2}} \det([\X\W^A\X^\top]_{\{a,b\}})} \cdot q(S),
\end{align*}
where the inequality comes from \cref{thm-ndpp-upper-bound}. This tells us that the average number of rejections is equal to
\begin{align} \label{eq:eq-avg-num-rejections}
    \frac{ \sum_{\{a,b\} \in \binom{[n]\setminus A}{2}} \det( [\X \widehat{\W}^A \X^\top]_{\{a,b\}})}{ \sum_{ \{a,b\} \in \binom{[n]\setminus A}{2}} \det( [\X {\W}^A \X^\top]_{\{a,b\}})}.
\end{align}
Instead of finding an upper bound on the above directly, we consider the following
\begin{align*}
	 \max_{ \{a,b\} \in \binom{[n]\setminus A}{2}} \frac{ \det( [\X \widehat{\W}^A \X^\top]_{\{a,b\}})}{ \det( [\X {\W}^A \X^\top]_{\{a,b\}})},
\end{align*}
which is greater than or equal to expression \eqref{eq:eq-avg-num-rejections}.

Now, for any symmetric and positive semidefinite (SPSD) matrix $\M$, we denote by $\lambda_{\max}(\M)$ and $\lambda_{\min}(\M)$ the largest and smallest nonzero eigenvalues of $\M$, respectively.
Let $\S \coloneqq \frac{\W^A + {\W^A}{}^\top}{2}$ and $\R \coloneqq \widehat{\W}^A - \S$. From the construction of  $\widehat{\W}^A$ in~\cref{eq-spec-sym}, it is easy to check that both $\S$ and $\R$ are SPSD. First we claim that for any $Y \subseteq [n] \setminus A$, it holds that
\begin{align} \label{eq-lower-bound-xwx}
	\det([\X \W^A \X^\top]_Y) \geq \det([\X \S \X^\top]_Y).
\end{align}
This comes from the following. If $\det([\X \S \X^\top]_Y)=0$, the result is trivial due to $\det([\X \W^A \X^\top]_Y) \geq 0$ for all $Y$. Assume that $\det([\X \S \X^\top]_Y)\neq 0$, then
\begin{align}
    \frac{\det([\X \W^A \X^\top]_Y)}{\det([\X \S \X^\top]_Y)} 
    &= \frac{\det([\X \S \X^\top]_Y +  [\X (\W^A-\S) \X^\top]_Y)}{\det([\X \S \X^\top]_Y)}  \\
    &= \det\left(\I_{|Y|} + \underbrace{[\X \S \X^\top]_Y^{-\frac12}~\X_{Y,:}}_{\coloneqq \X'}~(\W^A - \S)~\X_{Y,:}^\top~[\X \S \X^\top]_Y^{-\frac12} \right) \\
    &= \det \left( \I_{|Y|} + \X' \left( \W^A - \S \right) \X'^\top\right) \\
    &= \sum_{T \subseteq [|Y|]} \det \left( \left[\X' \left( \W^A - \S \right) \X'^\top\right]_T\right) \\
    &\ge \det ( \left[\X' \left( \W^A - \S \right) \X'^\top\right]_{\emptyset}) = 1,
\end{align}
where the fourth line comes from \citep[Theorem 2.1]{kulesza2012determinantal}, and the last line follows from the observation that $\W^A - \S=\frac{\W^A-\W^A{}^\top}{2}$ is a skew-symmetric matrix, so that every principal submatrix has a nonnegative determinant. Now we fix some $\{a,b\} \in \binom{[n]\setminus A}{2}$ and denote $\Q := \X_{\{a,b\},:} \in \Rbb^{2 \times d}$. Then we have
\begin{align}
	 \frac{ \det( [\X \widehat{\W}^A \X^\top]_{\{a,b\}})}{ \det( [\X {\W}^A \X^\top]_{\{a,b\}})}
	 &\le \frac{\det([\X \widehat{\W}^A \X^\top]_{\{a,b\}}}{\det([\X \S \X^\top]_{\{a,b\}})} \\
	 &= \frac{\det(\Q (\S + \R) \Q^\top)}{\det(\Q \S \Q^\top)} \\
	 &= \det\left(\I_2 + \left( \Q \S \Q^\top\right)^{-1/2} \Q \R \Q^\top  \left( \Q \S \Q^\top\right)^{-1/2} \right) \\
	  &\le \left( \frac{1}{2}\cdot\tr\left(\I_2 + \left( \Q \S \Q^\top\right)^{-1/2} \Q \R \Q^\top  \left( \Q \S \Q^\top\right)^{-1/2}\right) \right)^2\\
	  &= \left(1 + \frac{1}{2} \cdot \tr\left( \Q \R \Q^\top  \left( \Q \S \Q^\top\right)^{-1} \right)\right)^2 \\
	  &\le \left(1 + \frac{1}{2} \cdot \tr \left( \Q \R \Q^\top  \right) \cdot  \lambda_{\max}\left( \left(\Q \S \Q^\top\right)^{-1} \right) \right)^2 \\
	  &\le \left(1 + \frac{ \lambda_{\max} \left(\Q \R \Q^\top \right) }{ \lambda_{\min}\left(\Q \S \Q^\top\right)}\right)^2,  \label{eq-condition-number-bound}
\end{align}
where the first line follows from \cref{eq-lower-bound-xwx}, the fourth line is due to the fact that $\det(\M)\le (\tr(\M)/d)^d$ for a SPSD matrix $\M\in\Rbb^{d \times d}$ (thanks to the \emph{AM-GM inequality}), the fifth line comes from the cyclic property of a trace, and the sixth line is from the fact that $\tr(\M\N)\leq \tr(\M) \cdot \lambda_{\max}(\N)$ for SPSD matrices $\M,\N$. 
For an arbitrary vector $\v \in \Rbb^{2}$, we observe that
\begin{align*} 
	\v^\top \left( \Q \R \Q^\top \right) \v \leq \lambda_{\max}(\R) \cdot \v^\top \Q \Q^\top \v \leq  \lambda_{\max}(\R)  \cdot \lambda_{\max}(\Q \Q^\top) \cdot \norm{\v}_2^2.
\end{align*}
Since $\Q\Q^\top = \X_{\{a,b\},:} \X_{\{a,b\},:}^\top = [\X\X^\top]_{\{a,b\}}\in \Rbb^{2 \times 2}$ is a principal submatrix of $\X \X^\top$, by Cauchy’s interlace theorem, all eigenvalues of $\Q\Q^\top$ interlace those of $\X \X^\top$, and thus $\lambda_{\max}(\Q\Q^\top) \le \lambda_{\max}(\X\X^\top) = \sigma_{\max}(\X)^2$. Furthermore, since the matrix $\R$ is obtained from the spectral symmetrization of $\frac{\W^A - \W^A{}^\top}{2}$, their spectra are identical, i.e., $\lambda_{\max}(\R) = \frac{\sigma_{\max}(\W^A - \W^A{}^\top)}{2}$. Therefore,
\begin{align} \label{eq-condition-number-upper}
    \lambda_{\max}\left(\Q\R\Q^\top\right) \leq \frac{\sigma_{\max}(\W^A - \W^A{}^\top)}{2} \cdot \sigma_{\max}(\X)^2.
\end{align}
In addition, we have\footnote{One can similarly show that $\lambda_{\min}\left(\Q \S \Q^\top\right) \ge \lambda_{\min}\left(\S\right) \cdot \lambda_{\min}\left(\Q\Q^\top\right) $. However, the matrix $\S$ can be rank-deficient, because $\W^A$ is computed by projecting $\W$ onto some subspace with dimension $d - |A|$. Thus, this approach gives us a trivial lower bound of zero.}\begin{align} \label{eq-condition-number-lower}
    \lambda_{\min}\left(\Q \S \Q^\top\right)
    = \frac{\sigma_{\min}\left( \left[\X\left(\W^A + {\W^A}{}^\top\right)\X^\top\right]_{\{a,b\}}\right)}{2} 
    \ge \frac{\min_{Y \in \binom{[n]\setminus A}{2}}\sigma_{\min}([\X(\W^A+\W^{A}{}^\top)\X^\top]_Y)}{2}.
\end{align}
Putting \cref{eq-condition-number-upper,eq-condition-number-lower} into \cref{eq-condition-number-bound} gives 
\begin{align}
	 \frac{ \det( [\X \widehat{\W}^A \X^\top]_{\{a,b\}})}{ \det( [\X {\W}^A \X^\top]_{\{a,b\}})} \leq 
	 \left( 1 + \sigma_{\max}(\X)^2 \cdot \kappa_A\right)^2,
\end{align}
where, in \cref{def-condition-number}, $\kappa_A$ is defined as 
\begin{align*}
    \kappa_A := \frac{\sigma_{\max}(\W^A-\W^{A^\top})}{\min_{Y \in \binom{[n]\setminus A}{2}}\sigma_{\min}([\X(\W^A+\W^{A}{}^\top)\X^\top]_Y)}.    
\end{align*}
This completes the proof of \cref{thm-num-rejections-bound}.
\end{proof}

\subsection{Proof of \cref{thm-total-runtime}} \label{sec-proof-thm-total-runtime}

\textbf{Proposition 5.}~~{\it Given $\X\in\Rbb^{n \times d}$ and $\W \in \Rbb^{d \times d}$, such that $\W + \W^\top \succeq 0$ and $k\geq 2$,  consider $\kappa$ as defined in \cref{def-condition-number}. With a preprocessing step that runs in time $\bigo(nd^2)$,
\cref{alg-down-up-kndpp} runs in time $\bigo(\tmix~(1+\sigma_{\max}(\X)^2~\kappa)^2~(d^2 \log n + d^3))$ in expectation.}

\begin{proof}
We remind the reader that our MCMC sampler (\cref{alg-down-up-kndpp}) repeatedly runs tree-based rejection sampling for $\tmix$ iterations. 
From \cref{thm-num-rejections-bound},  each iteration requires  $2$-DPP sampling for at most $(1 + \sigma_{\max}(\X)^2 \kappa)^2$ times on average.  From \cref{thm-runtime-tree-based-kdpp}, sampling from the $2$-DPP can be done in time $\bigo(d^2 \log n + d^3)$. Combining all of these runtimes gives the result.
\end{proof}

\end{document}